%% file: main.tex
\pdfoutput=1
\documentclass[journal,onecolumn,draftclsnofoot,twoside]{IEEEtran} 

\include{preamble}

\usepackage{microtype}
\usepackage{graphicx}
\usepackage{subfigure}
\usepackage{booktabs} 

\begin{document}

\title{Global Multiclass Classification and Dataset Construction via Heterogeneous Local Experts}
	
\author{Surin~Ahn,
			Ayfer~{\"O}zg{\"u}r,
			and Mert~Pilanci

	\thanks{Surin~Ahn, Ayfer~{\"O}zg{\"u}r, and Mert~Pilanci are with the Department of Electrical Engineering, Stanford University, Stanford, CA 94305 USA (e-mail: surinahn@stanford.edu; aozgur@stanford.edu; pilanci@stanford.edu).}}

\maketitle

\begin{abstract}
	In the domains of dataset construction and crowdsourcing, a notable challenge is to aggregate labels from a heterogeneous set of labelers, each of whom is potentially an expert in some subset of tasks (and less reliable in others). To reduce costs of hiring human labelers or training automated labeling systems, it is of interest to minimize the number of labelers while ensuring the reliability of the resulting dataset. We model this as the problem of performing $K$-class classification using the predictions of smaller classifiers, each trained on a subset of $[K]$, and derive bounds on the number of classifiers needed to accurately infer the true class of an unlabeled sample under both adversarial and stochastic assumptions. By exploiting a connection to the classical \textit{set cover} problem, we produce a near-optimal scheme for designing such configurations of classifiers which recovers the well known one-vs.-one  classification approach as a special case. Experiments with the MNIST and CIFAR-10 datasets demonstrate the favorable accuracy (compared to a centralized classifier) of our aggregation scheme applied to classifiers trained on subsets of the data. These results suggest a new way to automatically label data or adapt an existing set of local classifiers to larger-scale multiclass problems. 
\end{abstract}

\input{intro}

\input{problem_formulation}
\input{perfect_accuracy}

\input{set_covering}
\input{statistical}
\input{experiments}

\input{discussion}

\input{related_work}

\section*{Acknowledgement}
This work was supported in part by the National Science Foundation under Grant IIS-1838179, Grant ECCS-2037304, and Grant NeTS-1817205; Facebook Research; Adobe Research; Stanford SystemX Alliance; a Google Faculty Research Award; and a Cisco Systems Stanford Graduate Fellowship.

The authors thank the JSAIT reviewers for their invaluable feedback, particularly for their suggestion to connect this work to the dataset construction and crowdsourcing literatures. S.A. would also like to thank Wei-Ning Chen, Kfir Dolev, and Huseyin A. Inan for helpful discussions during the research process. 

\bibliography{references}
\bibliographystyle{IEEEbib}

\newpage 
\input{appendix}

\end{document}

%% file: preamble.tex
\usepackage{amsmath,amsfonts,amssymb} 
\usepackage{amsthm} 

\usepackage{color}
\usepackage{mathtools}
\usepackage{hhline}
\usepackage{verbatim}
\usepackage[dvipsnames]{xcolor}

\usepackage[hyphens]{url}
\usepackage{hyperref}
\hypersetup{
	colorlinks=true,
	linkcolor=red,
	filecolor=magenta,      
	urlcolor=cyan,
	breaklinks=true
}

\newtheorem{definition}{Definition}
\newtheorem{theorem}{Theorem}
\newtheorem{lemma}{Lemma}
\newtheorem{corollary}{Corollary}

\newcommand{\PP}{\mathbb{P}}

\newcommand{\RR}{\mathbb{R}}

\newcommand{\calC}{\mathcal{C}}
\newcommand{\calD}{\mathcal{D}}
\newcommand{\calF}{\mathcal{F}}

\newcommand{\calL}{\mathcal{L}}

\newcommand{\calS}{\mathcal{S}}
\newcommand{\calU}{\mathcal{U}}

\newcommand{\calX}{\mathcal{X}}
\newcommand{\calY}{\mathcal{Y}}

\newcommand{\Sum}[2]{\sum\limits_{#1}^{#2}}
\newcommand{\Prod}[2]{\prod\limits_{#1}^{#2}}

\DeclarePairedDelimiter\ceil{\lceil}{\rceil}


%% file: intro.tex
\section{Introduction}
In modern machine learning systems and research, one of the primary bottlenecks is the availability of high quality datasets for training and evaluating models. Indeed, a growing body of literature concerns the question of how to properly construct datasets \cite{ratner2017snorkel, deng2009imagenet, russakovsky2015imagenet, collins2008towards} or label large-scale data via crowdsourcing \cite{karger2014budget, NIPS2011_c667d53a, karger2013efficient}. Recent works \cite{recht2019imagenet, Xie_2020_CVPR} have even called into question the ability of the ImageNet dataset -- a standard benchmark for classification models over the past several years -- to produce models that truly generalize. As a result, datasets are increasingly being viewed as dynamic entities that need to be continuously updated and improved. 

Often, datasets for supervised learning tasks are laboriously hand-labeled by human experts (e.g., medical professionals who examine X-rays or MRI scans) or by labelers hired through crowdsourcing platforms such as Amazon Mechanical Turk \cite{mturk}. Human labelers are prone to error (particularly when labeling samples outside their realm of expertise), and the overall data labeling process can be expensive and time consuming. Therefore, it is of great interest to
\begin{enumerate}
	\item  \textit{minimize} the number of labelers while still ensuring that the resulting dataset is labeled accurately, and 
	\item \textit{automate} the data labeling process as much as possible.
\end{enumerate}

In this paper, we take a step toward achieving these goals by studying, through an information-theoretic lens, the general problem of \textit{ensembling} or \textit{aggregating} a set of smaller, heterogeneous classifiers to solve a larger classification task. Specifically, we ask the fundamental question: 
\medskip 

\noindent \textit{What is the optimal way to construct a global $K$-class classifier, given black-box access to smaller $R$-class classifiers, where $R \in \{2,\ldots, K\}$, and knowledge of which classifiers specialize in which subsets of classes}?
\medskip 

\noindent In the context of goal (1) above, each smaller classifier models a human labeler who is an ``expert'' in some subset of tasks, but gives unreliable or noisy labels for other tasks. As a concrete example, suppose we are interested in diagnosing a patient based on their symptoms using the assessments of medical experts, each of whom specializes in a different set of health conditions. It is plausible that each expert correctly diagnoses the patient if the patient's true condition falls within their area of expertise, but otherwise provides an incorrect diagnosis. 
Before collecting their opinions, we might ask each medical expert to self-report their areas of expertise, and to diagnose the patient \textit{using only these labels} (for example, a cancer specialist should not claim that a patient has heart disease). We are interested in 1) identifying the right set of heterogeneous experts, and 2) aggregating their opinions to infer the patient's true condition, while taking into account which specialists are reliable sources in which areas. In the case of dataset labeling, we are similarly interested in aggregating the votes of various human labelers -- using \textit{a priori} information about their domain knowledge -- to accurately label a given data sample. 

To address goal (2) above, we envision generating new datasets for large-scale classification tasks by aggregating the labels of existing classifiers which specialize in different subsets of classes. In this work, we also show how to adapt an existing set of classifiers to a broader multiclass problem by carefully selecting additional local classifiers. 
Moreover, our work contributes to the literature on \textit{ensemble methods} for multiclass classification. The traditional approach to multiclass classification is to either train a single centralized classifier (e.g., a neural network) on all $K$ classes, or reduce the problem to multiple binary classification tasks. 
This paper explores the uncharted region between these two extremes, providing a generalization of the standard ``one-vs.-one'' decomposition method \cite{Friedman1996AnotherAT, hastie1998classification} in which a data sample is labeled using the majority vote of all possible ${K \choose 2}$ pairs of binary classifiers. A real-world example where our approach might be useful is in identifying diseases present in a biological sample using tests or prediction models which specialize in detecting different subsets of diseases, and which may have been provided by separate hospitals or silos. Furthermore, the classifiers may have been trained using different algorithms and architectures, and therefore must be treated as black-boxes.

Though the model we study in this paper is stylized, it nevertheless yields some fundamental insights into how one should select and aggregate data labelers or classifiers to perform large-scale dataset creation or classification tasks. By assuming that each local classifier is an expert within its own domain, we are effectively investigating the bare minimum requirements that the local classifiers must satisfy to form an accurate global classifier. In practice, one certainly does not expect these classifiers to be perfectly accurate, and it is most likely helpful to judiciously introduce  redundancy into the set of classifiers to improve the global accuracy. However, we view our work as an initial step toward the rigorous study of aggregating heterogeneous experts.

\subsection{Contributions}
We summarize the main contributions of this paper as follows: 
\begin{itemize}
	\item We propose a mathematical model for studying the ensembling of smaller $R$-class classifiers to perform $K$-class classification, 
	and justify this model through empirical findings.  
	\medskip 

	\item Under this model, we derive necessary and sufficient conditions for achieving perfect global classification accuracy under adversarial (worst-case) noise within the local classifiers, and derive bounds which  scale as $\Theta(K^2/R^2)$ (up to a $\log$ factor) on the number of smaller classifiers required to satisfy these conditions. Moreover, we show that a random set of classifiers satisfies the conditions with high probability.
	\medskip 
	
	\item We introduce an efficient voting-based decoding scheme for predicting the true label of an input given the predictions of the smaller classifiers.
	\medskip 
	
	\item We show that the conditions for perfect accuracy are intrinsically related to the classical set cover problem from the combinatorics and theoretical computer science literatures. This connection leads to near-optimal algorithms for designing configurations of smaller classifiers to solve the global problem. We also introduce variations of the original set cover algorithms to address questions specific to dataset construction and adapting a set of local classifiers to larger multiclass objectives.
	\medskip 
	
	\item We consider a statistical setting in which we assume a uniform prior over the classes, uniformly distributed noise within each classifier, and allow a small probability of misclassification, and show that the required number of smaller classifiers now scales as $\Theta(K/R)$ (up to a $\log$ factor), which is a significant reduction compared to the perfect accuracy case.
	\medskip 
	
	\item Through experiments with the MNIST and CIFAR-10 datasets, 
	we show that our set covering-based scheme is able to match the performance of a centralized classifier, and demonstrate its robustness to factors such as missing local classifiers. 
\end{itemize}

%% file: problem_formulation.tex
\section{Problem Formulation}\label{sec:model}

We consider black-box access to a set of $m$ classifiers $\calF_{m,K} = \{f_i: \calX \to \calY_i, \, i \in [m] \}$, with $\calY_i \subseteq [K]\triangleq \{1,2,\ldots, K\}$ and $2 \leq |\calY_i| \leq K$ for all $i \in [m]$. Each possible input $x\in\calX$ belongs to one of the classes in $[K]$. We assume that each classifier $f_i$ was trained to distinguish only between a subset of the classes in $[K]$, i.e., those contained in $\calY_i$. Therefore, given an input $x \in \calX$, each $f_i$ necessarily outputs a class in $\calY_i$. Any classes in $[K] \, \backslash \, \calY_i$ are considered to be outside of its ``universe.'' This models a distributed setting, where each classifier or local expert has access to data belonging to only a subset of the $K$ classes. Note that $f_i$ outputs only its final class prediction, rather than confidence scores or conditional probability estimates (we later provide some justification for this modeling decision). We further assume that  $\{\calY_i, \, i \in [m]\}$ is known to the central orchestrator, and we call $|\calY_i|$ the \textit{size} of the classifier $f_i$. Given a new input $x_k \in \calX$ belonging to class $k \in [K]$, we assume that
\[f_i(x_k) = k \, \text{ if } k \in \calY_i.\]
In words, $f_i$ always makes correct predictions on inputs belonging to familiar classes. This model captures the notion that a properly trained classifier or labeler is expected to accurately classify a new input whose true class is among its known universe of classes.  When $k \not\in \calY_i$, then by definition $f_i(x_k) \neq k$ since $f_i$ always maps to $\calY_i$. Hence, $f_i$'s predictions on inputs belonging to classes outside of $\calY_i$ can be considered as undesirable ``noise'' that we wish to circumvent. We will consider both adversarial (or worst-case) and stochastic models for  $f_i(x_k)$ when $k \not\in \calY_i$.

In this paper, we study the problem of inferring the class of an unknown input given the ``one-shot'' outputs of the $m$ classifiers $\calF_{m,K}$, i.e., the results of feeding the input a single time to each of the classifiers. Note that this one-shot modeling assumption is fitting in practice, as trained classifiers typically produce the same output when given the same input multiple times. We next define a \textit{$K$-class classification scheme} constructed out of $\calF_{m,K}$. 
	
	\begin{definition}[Classification Scheme]
		A $K$-class classification scheme is a pair $(\calF_{m,K}, g)$ where $\calF_{m,K} = \{f_i: \calX \to \calY_i, \, i \in [m] \}$ is a set of $m$ local classifiers satisfying $\calY_i \subseteq [K]$, $2 \leq |\calY_i| \leq K$ for all $i \in [m]$, and $g: \calY_1\times\cdots \times \calY_m \to [K]$ is a decoder which predicts a class upon observing the outputs produced by classifiers in $\calF_{m,K}$. Specifically, given an input $x \in \calX$, the global class prediction is given by $g(f_1(x),f_2(x),\ldots,f_m(x))$. 
	\end{definition}

We remark that this problem can be interpreted as an unorthodox communications setting where the transmitter is trying to convey a message $k \in [K]$ to the receiver, but the receiver can only observe the outputs of $m$ channels, each of which is selectively noisy depending on whether $k$ is among its ``accepted'' symbols. The goal of the receiver is to decode the message $k$ given the $m$ channel outputs.

Broadly, our goal in this paper is to study when and how we can construct an accurate $K$-class classification scheme given $\calF_{m,K}$, or conversely, how we can construct a set $\calF_{m,K}$ of $m$ small classifiers of a given size so that accurate $K$-class classification using $\calF_{m,K}$ is possible. In the second case, it is clearly desirable for $\calF_{m,K}$ to be minimal. Note that the first problem corresponds to synthesizing a global (bigger) classifier from a given set of local (smaller) classifiers, while the second problem corresponds to decomposing a global classifier into multiple local classifiers. In the rest of the paper, we will study these problems in the following two different settings.

\subsection{Perfect Accuracy Setting}

Here, we will require the $K$-class classification scheme $(\calF_{m,K}, \calD)$ to correctly recover the true class of any input $x\in\calX$ for any possible set of outputs from $\calF_{m,K}$. More precisely, we define the \textit{output set } $\calS_k$ of a class $k \in [K]$, with respect to a fixed set of classifiers $\calF_{m,K} = \{f_i: \calX \to \calY_i, \, i \in [m]\}$, as the set of all possible classifier outputs given that the true class of the input is $k$: 

{
\begin{equation*}
\calS_k \triangleq \Bigg\{(y_1,\ldots, y_m) \, : \, y_i = k \text { if } k \in \calY_i, \, y_i \in \calY_i \text{ if } k \in [K]\backslash \calY_i \Bigg\}.
\end{equation*}
}

Note that  $\calS_k$ can be constructed using only knowledge of the $\calY_i$. In the perfect accuracy setting, we require the $K$-class classification scheme to correctly recover $k$ given any observation $y=(y_1,\ldots, y_m)\in \calS_k$. Specifically, we say that a scheme $(\calF_{m,K}, g)$ achieves perfect $K$-class classification accuracy if for any $k\in [K]$ and any $y \in \calS_k$, we have $g(y) = k$.

Note that this can be regarded as an adversarial or worst-case setting in the sense that for each class $k \in [K]$, the $m$ classifiers are allowed to jointly produce the most confusing output $y=(y_1,\ldots,y_m)$. In particular, if perfect accuracy can be achieved under this model, it can be achieved under any joint probabilistic model for the behavior of the $m$ classifiers. In the next section, we focus on one specific probabilistic model, which assumes that the outputs of the $m$ classifiers are independent and uniformly distributed over $\calY_i$ when $k\not\in\calY_i$.

\subsection{Statistical Setting}\label{subsec:statistical}

In this setting, given a new test sample $x_k \in \calX$ belonging to class $k \in [K]$, we will assume that a local classifier $f_i: \calX \to \calY_i$  correctly outputs $k$ if $k \in \calY$, and otherwise will pick a class uniformly at random among those in $\calY_i$. Mathematically, 
\begin{equation}\label{eqn:model}
f_i(x_k) = 
\begin{cases}
k, & \text{if } k \in \calY_i \\
U \sim \text{Uniform}(\calY_i), & \text{if } k \not\in \calY_i
\end{cases}
\end{equation}
where $U \sim \text{Uniform}(\calY_i)$ denotes a uniform random variable with support $\calY_i$.  Note that the output of a classifier in this setting (as well as the earlier perfect accuracy setting) depends only on the true class $k$ corresponding to an input $x_k$, and does not depend on the input itself. Therefore, we will sometimes write $f(k)$ instead of $f(x_k)$ for notational simplicity. 

Given $m$ local classifiers $f_i : \calX \to \calY_i, \, i \in [m]$, we assume that the outputs of distinct classifiers~\footnote{We say that classifiers $f_1$ and $f_2$ are distinct if $\calY_1 \neq \calY_2$.}, denoted by the random vector $Y = (Y_1,\ldots, Y_m) \in \calY_1 \times \calY_2 \times \cdots \times \calY_m$, are conditionally independent given the true class of the input, denoted by $Z$. This is equivalent to assuming that two distinct classifiers $f_1,f_2$ have independent noise, i.e., independent sources of randomness $U_1, U_2$. In this setting, we further assume that $Z$ is chosen uniformly at random from $[K]$, i.e., the prior we impose on the classes is given by $\pi(k) = \frac{1}{K}, \forall k \in [K]$. Let $P_e \triangleq \PP(g(Y) \neq Z)$ denote the average probability of error, where $g(Y)$ is the decoder's estimate of $Z$ based on $Y$. We will now require the $K$-class classifier to have $P_e \leq \epsilon$ for some fixed $\epsilon \in (0,1)$. In the sequel, we aim to understand how the decoder can exploit the fact that the true class and the noisy outputs of the classifiers are uniformly distributed and whether this can lead to significant gains with respect to the worst-case setting discussed earlier.

\subsection{Model Justification}\label{sec:justification}

While our proposed classification model is stylized, it offers a number of benefits. First, it is fairly general, as it makes no assumptions about the underlying classification algorithm. We consider only ``hard'' outputs (i.e., final class predictions), rather than ``soft'' outputs or confidence scores (e.g., outputs of a softmax layer or estimates of conditional class probabilities), as the soft outputs may be incomparable across heterogeneous classifiers due to differences in units or in the underlying algorithms and architectures. In the case of human labelers, it is even less clear how one could obtain comparable soft outputs. 

\begin{figure}[t]
	\centering
	\subfigure[]{\includegraphics[width=0.3\textwidth]{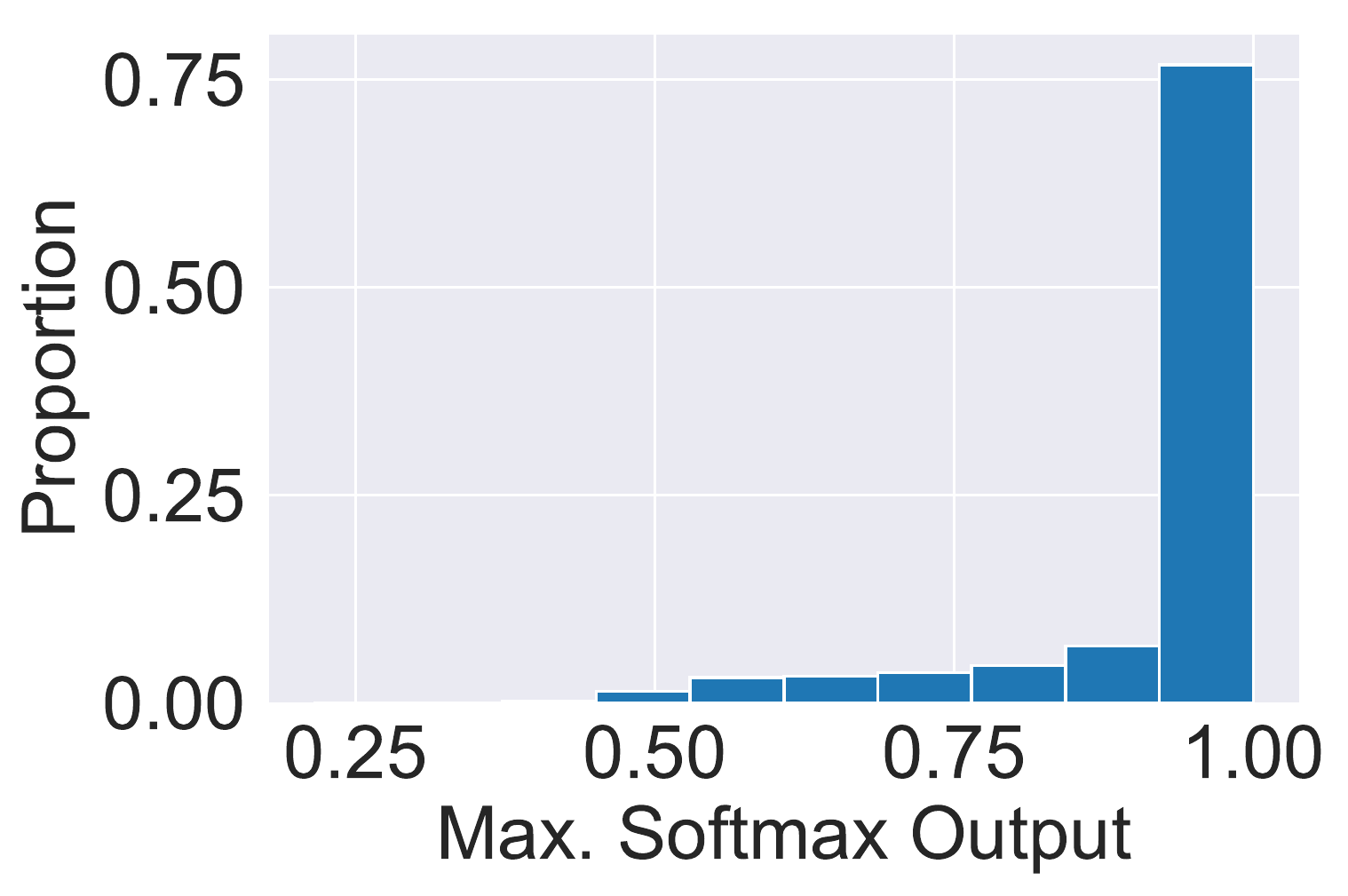}}\hspace{8em} 
	\subfigure[]{\includegraphics[width=0.3\textwidth]{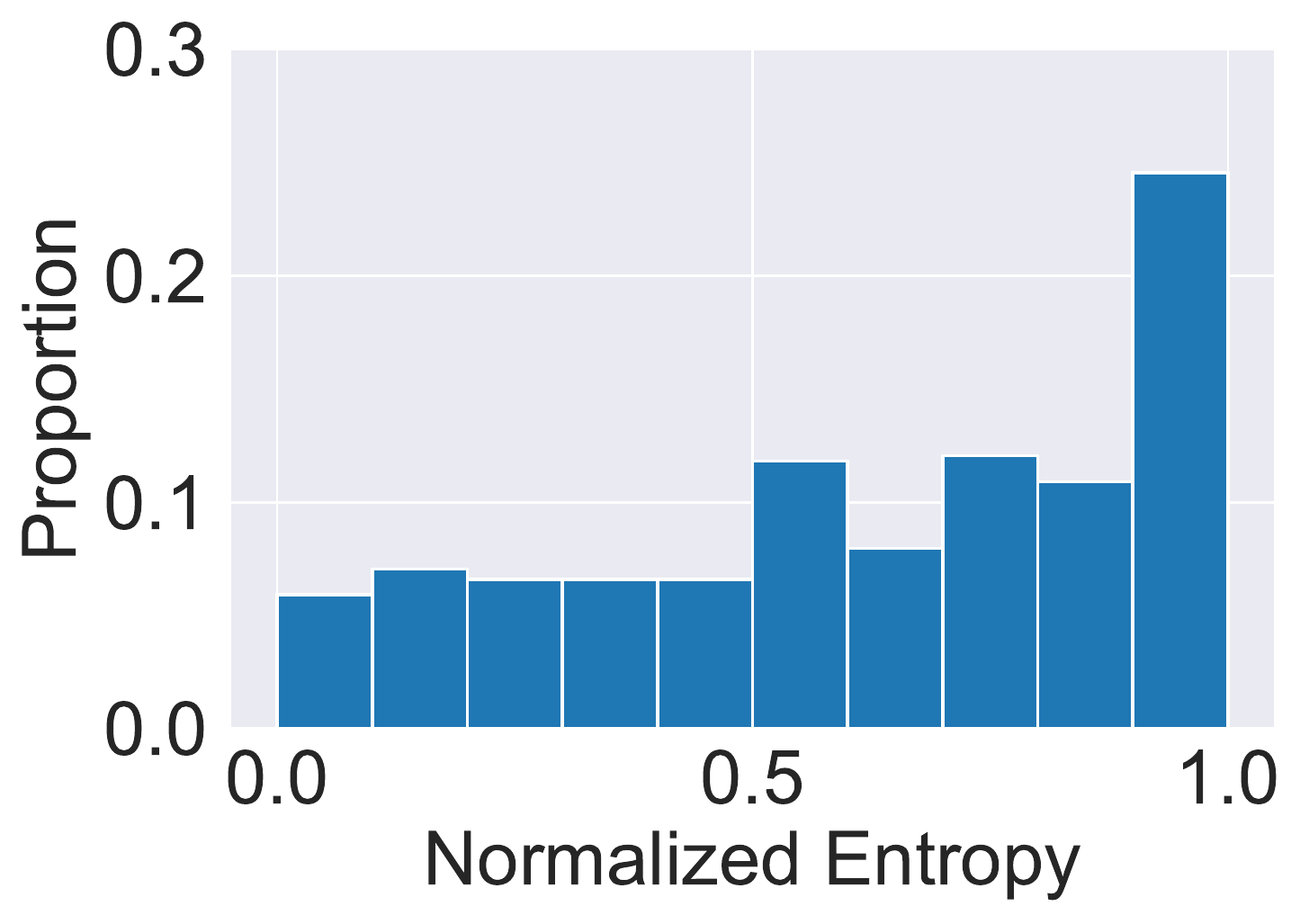}}
	\caption{Histograms of (a) the maximum softmax output of MNIST classifiers when given images belonging to an unfamiliar class, and (b) the normalized entropy of the empirical distribution of predictions produced by classifiers when given images of an unfamiliar class.}
	\label{fig:model_justification}
\end{figure}

Figure~\ref{fig:model_justification}(a) provides further justification for assuming hard outputs. 
We trained 62 different convolutional neural networks (CNNs) with a softmax layer as the final layer, each trained on a different subset of the 10 MNIST digits \cite{lecun1998gradient} 
(see Section~\ref{sec:experiments} for further details). For each classifier, we observed its predictions on 1,000 images of each \textit{unfamiliar class}, and plotted the resulting histogram of maximum softmax outputs in  Figure~\ref{fig:model_justification}(a). The histogram suggests that in practice, classifiers often give extremely confident predictions of images for which they know nothing about the underlying class, which supports our hard-outputs modeling decision.  A related discussion was provided recently in the context of adversarial machine learning \cite{verma2019error}.

We now provide some partial justification for assuming in the statistical setting that the local classifiers have uniformly distributed noise. In the proof of Theorem~\ref{thm:stat_LB}, we will see that our lower bound on the minimum number of classifiers needed to achieve $P_e \leq \epsilon$ is maximized by minimizing the mutual information between the true class and the classifier outputs. This is achieved by maximizing the conditional entropy of predictions given the true class, which in turn is achieved by the uniform distribution. Thus, one might suspect that uniform noise can make the problem more difficult compared to other random models, assuming that classifier outputs are conditionally independent given the true class. However, we acknowledge that our lower bound in general does not provide sufficient evidence to conclude that uniform noise is ``worst-case'' in a stochastic sense.

Empirically, we also observed that classifiers sometimes exhibit behavior similar to uniform noise. Using the same setup as before, for each classifier and each unfamiliar class we generated the empirical distribution, $\hat{\mathcal{P}} = \begin{pmatrix}\hat{p}_1,\ldots,\hat{p}_R\end{pmatrix}$ -- where $R$ is the size of the classifier -- of labels assigned by the classifier when given images belonging to the unfamiliar class. We then computed the normalized entropy \[\frac{1}{\log (R)}\cdot H(\hat{\mathcal{P}}) = -\frac{1}{\log (R)}\Sum{i=1}{R}\hat{p}_i\log \hat{p}_i. \] 
Figure~\ref{fig:model_justification}(b) shows the histogram of normalized entropies that were observed. A normalized entropy close to 1 indicates that the distribution $\hat{P}$ is close to the uniform distribution with the same support. Indeed, the histogram shows that a nontrivial proportion of the MNIST classifiers exhibited behavior similar to the uniform noise in our model.

\section{Paper Organization}\label{sec:organization} 
Throughout the rest of this paper, we answer the following questions. Section \ref{sec:perfect} addresses questions (1) and (2), Section \ref{sec:setcover} addresses question (3), and Section \ref{sec:statistical} addresses question (4).
\begin{enumerate}
	\item What conditions must classification schemes satisfy for perfect $K$-class classification accuracy to be possible?
	\medskip 
	
	\item  Suppose we can design our own set of $m$ classifiers $\{f_i: \calX \to \calY_i \}$, under the size constraint $|\calY_i| \leq R, \, \forall i \in [m]$, for some fixed integer $R$ satisfying $2 \leq R \leq K$. What is the minimum number of classifiers, $m^*$ (in terms of $K$ and $R$), needed to satisfy the conditions from (1)? 
	\medskip 
	
	\item  Is there an efficient algorithm for designing these classifiers given $K$ and $R$? 
	\medskip 
	
	\item In the statistical setting, how does the scaling of $m^*$ change? 
\end{enumerate}

Subsequently, we present experimental results in Section~\ref{sec:experiments}, provide a discussion and potential extensions of our work in Section~\ref{sec:discussion}, and close with a summary of related works in Section~\ref{sec:related}. All omitted proofs are given in the Appendix.

%% file: perfect_accuracy.tex
\section{Perfect Accuracy Setting}\label{sec:perfect}
We begin by answering question (1). Here, the goal is to determine the conditions under which exact recovery of the true class corresponding to an input is possible, based solely on observations of the classifier outputs and \textit{a priori} knowledge of $\calY_1,\ldots \calY_m$.

\subsection{Necessary and Sufficient Condition}\label{subsec:conditions}
First, we define the notion of distinguishability between two classes. 
\begin{definition}[Distinguishability]
	We say that two classes $k, k' \in [K]$, $k \neq k'$ are distinguishable with respect to a set of classifiers $\calF_{m,K}$ if their output sets are disjoint, i.e., $\calS_k \cap \calS_{k'} = \emptyset$.
\end{definition}

From a communications perspective, the vectors in a class's output set can be interpreted as the possible channel outputs that the decoder may observe after the class is transmitted over the $m$ noisy channels. Thus, a natural claim is that perfect accuracy is achievable when the output sets of the $K$ classes do not overlap with each other, i.e., each pair of classes is distinguishable. If a set of classifiers $\calF_{m,K}$ satisfies these conditions, we say that it achieves \textit{pairwise distinguishability}. Theorem~\ref{thm:distinguishability} below says that pairwise distinguishability is a necessary and sufficient condition for achieving perfect accuracy under our classifier model. 

\begin{theorem}\label{thm:distinguishability}
	Given a set of $m$ classifiers $\calF_{m,K}=\{f_i: \calX \to \calY_i, i \in [m]\}$ with $\calY_i \subseteq [K]$ and $2 \leq |\calY_i| \leq K$ for all $i \in [m]$, perfect accuracy for the $K$-class classification problem is achievable if and only if all pairs of classes are distinguishable with respect to $\calF_{m,K}$. 
\end{theorem}
\begin{proof}
	First, suppose that all pairs of classes are distinguishable, and consider the following decoding function, $g$, for predicting the class of a particular input given the $m$ observed classifier outputs. Initially, we generate a lookup table consisting of $\calS_1,\ldots, \calS_K$. Note that this can be done using only $\calY_1,\ldots, \calY_m$, which we assume are known \textit{a priori}. Given an input $x \in \calX$, we observe the classifier outputs $y = (f_1(x),f_2(x),\ldots, f_m(x))$. We then simply find any $\calS_{\hat{k}}$ such that $y \in \calS_{\hat{k}}$, and declare $\hat{k}$ as the prediction, i.e., set $g(y) = \hat{k}$. If $k \in [K]$ is the true class, then it must be the case that $y \in \calS_k$, by definition of the output sets. Moreover, by the assumption of pairwise distinguishability, $\calS_k$ must be the \textit{only} output set containing $y$. Thus, $g(y) = k$.
	
	Now, suppose w.l.o.g. that classes $1,2 \in [K]$  are not distinguishable with respect to $\calF_{m,K}$. This means $\exists \, y = (y_1,\ldots, y_m)$ such that $y \in \calS_{1} \cap \calS_{2}$. For any decoder $g$, note that $g(y) \neq 1$ or $g(y) \neq 2$, as $g(y) = 1 \implies g(y) \neq 2$. If $g(y) \neq 1$, then since $y \in \calS_1$, the classification scheme $(\calF_{m,K},g)$ fails to achieve perfect accuracy. Similarly, failure occurs if $g(y) \neq 2$,  since $y \in \calS_2$.
		
\end{proof}

Theorem~\ref{thm:distinguishability} says that as long as there is zero ambiguity in the outputs that can result from each of the $K$ classes, i.e., $\calF_{m,K}$ gives an injection from classes to classifier outputs, then the decoder can always determine the correct class. However, this result does not address whether pairwise distinguishability can be achieved by some set of classifiers $\calF_{m,K}$. The lemma below gives a condition that is  equivalent to distinguishability. 

\begin{lemma}\label{lem:distinguish}
	Two classes $k,k' \in [K]$, $k \neq k'$ are distinguishable with respect to $\calF_{m,K}$ if and only if there exists a classifier $f \in \calF_{m,K}, \, f : \calX \to \calY$, such that $k,k' \in \calY$. 
\end{lemma}
\begin{proof}
	
	First, suppose there exists a classifier $f_i: \calX \to \calY_i$ such that $k,k' \in \calY_i$. Then every $y=(y_1,\ldots,y_m) \in \calS_k$ satisfies $y_i = k$, and every $y' = (y_1',\ldots,y_m') \in \calS_{k'}$ satisfies $y'_i = k'$. It follows that $\calS_k \cap \calS_{k'} = \emptyset$, so $k$ and $k'$ are distinguishable.
	
	Now suppose there is no classifier $f: \calX \to \calY$ such that $k, k' \in \calY$. Then for every classifier $f: \calX \to \calY$, one of the following can happen: 1) $k, k' \not\in \calY$, 2) $k \in \calY$ and $k' \not\in \calY$, or 3) $k \not\in \calY$ and $k' \in \calY$. Consider the $i^\text{th}$ classifier, $f_i: \calX \to \calY_i$, and suppose case 1 holds. Then set $y_i = \tilde{k}$ for some arbitrary $\tilde{k} \in \calY_i$. For cases 2 and 3, assume w.l.o.g. that $k \in \calY_i$ and $k' \not\in \calY_i$. In this case, set $y_i = k$. The resulting vector of outputs $y = (y_1,\ldots,y_m)$ satisfies $y \in \calS_k \cap \calS_{k'}$, so $k$ and $k'$ are not distinguishable.  

\end{proof}

An immediate consequence of Theorem~\ref{thm:distinguishability} and Lemma~\ref{lem:distinguish} is the following corollary.

\begin{corollary}\label{cor:covering}
	Perfect accuracy under a set of classifiers $\calF_{m,K}$ is achievable if and only if for every pair of classes $k,k' \in [K], \, k \neq k'$, there exists a classifier $f \in \calF_{m,K}$, $f: \calX \to \calY$, such that $k,k' \in \calY$.
\end{corollary}
To achieve perfect accuracy under worst-case noise, it therefore suffices to have the local classifiers, in aggregate, cover all possible pairwise connections between classes. If this is the case, we say that $\calF_{m,K}$ satisfies the \textit{covering condition}. The question of how to algorithmically generate configurations $\calF_{m,K}$ satisfying the covering condition will be addressed in Section~\ref{sec:setcover}, where we discuss connections to the set cover and clique cover problems.  Corollary~\ref{cor:covering} also makes it clear that when given the choice to design a classifier of size \textit{at most} $R$ for some $R \in \{2,3,\ldots,K\}$, one should always choose the maximum possible size, $R$, as this can only help us get closer to achieving perfect accuracy. 

\subsection{Decoding Schemes}\label{subsec:decoder}
We now discuss decoding schemes for predicting the class of an input $x \in \calX$ given the $m$ observed classifier outputs $y = (f_1(x),f_2(x),\ldots, f_m(x))$. In the proof of Theorem~\ref{thm:distinguishability}, we considered a na{\"i}ve approach which requires generating a lookup table of all classes' output sets $\calS_1,\ldots, \calS_K$. When the covering condition from Corollary~\ref{cor:covering} is satisfied, we can  instead use a more efficient decoding scheme with complexity $O(mK)$, which works as follows  \footnote{We note that this decoding scheme can be used in practice even when the covering condition is not satisfied. The covering condition just ensures that perfect accuracy is achieved under our model.}. Instead of storing the entire lookup table, the decoder can store what we call the \textit{authority classifiers}, $\calC_1,\ldots, \calC_K$, defined for each class $k \in [K]$ as the set of classifiers which were trained on class $k$: 
\begin{equation}
\calC_k \triangleq \{i \in [m] \, : \, k \in \calY_i\}, \quad  k \in [K].
\end{equation}
We can think of $\calC_k$ as the indices of classifiers whose predictions we trust with respect to inputs belonging to class $k$. For each $k \in [K]$, the decoder counts the number of votes received from $k$'s authority classifiers: 
\begin{equation}
N_k(y) \triangleq \Big|\{i \in \calC_k\, : \, y_i = k\}\Big|.
\end{equation}
Finally, it predicts the class which received the largest normalized number of votes from its authority classifiers:
\begin{equation}\label{eqn:decoder}
g(y) = \underset{k \in [K]}{\text{argmax}}\; \frac{N_k(y)}{|\calC_k|}.
\end{equation}

Suppose $\calF_{m,K}$ satisfies the covering condition. If the true class of the input is $k$, then for any $y \in \calS_k$, it must be the case that $N_k(y)/|\calC_k| = 1$, as all of $k$'s authority classifiers will correctly output $k$. On the other hand, for any $k' \neq k$, by assumption there must exist a classifier $f: \calX \to \calY$ such that $k,k' \in \calY$. Note that $f$ is an authority classifier for both $k$ and $k'$, but will be guaranteed to output $k$. We will therefore observe $N_{k'}(y)/|\calC_{k'}| < 1$, and hence the decoding scheme will correctly predict $k$.

\subsection{Binary Matrix Representation} 
We now start to address question (2) from Section~\ref{sec:organization}. We find that there is a one-to-one correspondence between classifier configurations $\calF_{m,K}$ and binary matrices with row weight at least 2. This abstraction will somewhat simplify our analysis in Section \ref{subsec:bounds}. 

\begin{definition}[Classification Matrix]\label{def:matrix}
	The classification matrix $A$ corresponding to a set of classifiers $\calF_{m,K} = \{f_i: \calX \to \calY_i, i \in [m]\}$ is an $m \times K$ binary matrix with 
	\[A_{ij} = 
	\begin{cases}
	1, & \text{ if } j \in \calY_i \\
	0, & \text{ otherwise,}
	\end{cases}
	\]
	i.e., $A_{ij} = 1$ if and only if the $i^\text{th}$ classifier was trained on the $j^\text{th}$ class. Conversely, a binary matrix $A \in \{0,1\}^{m \times K}$ with row weight at least $2$ uniquely defines a set of classifiers $\calF_{m,K}$ as follows: the $i^\text{th}$ row of $A$ defines a classifier $f_i : \calX \to \calY_i$ with $\calY_i = \{j \in [K] \,:\, A_{ij} = 1 \}$.
\end{definition}

The following lemma provides a bridge between the classification matrix and the results of Section \ref{subsec:conditions}.

\begin{lemma}\label{lem:matrix}
	For a pair of classes $k,k' \in [K], \, k \neq k'$, there exists a classifier $f \in \calF_{m,K}, \, f: \calX \to \calY$, such that $k, k' \in \calY$ if and only if there exists an $i \in [m]$ such that $A_{ik} = A_{ik'} = 1$, where $A \in \{0,1\}^{m \times K}$ is the classification matrix corresponding to $\calF_{m,K}$. 
\end{lemma}
\begin{proof}
	This follows by construction of the classfication matrix $A$.
\end{proof}


\begin{lemma}\label{lem:fully_distinguishing}
	Perfect accuracy is achievable under a set of classifiers $\calF_{m,K}$ with corresponding classification matrix $A \in \{0,1\}^{m \times K}$ if and only if for any $k,k' \in [K]$ with $k \neq k'$, there exists an $i \in [m]$ such that $A_{ik} = A_{ik'} = 1$. In this case, we say that $A$ is \textbf{fully distinguishing.}
\end{lemma}

\begin{proof}
	Combine Corollary~\ref{cor:covering} and Lemma~\ref{lem:matrix}. 
\end{proof}

This result reformulates our exact-recovery problem as the somewhat more concrete problem of designing the binary classification matrix. To answer question (2) from Section~\ref{sec:organization}, we now impose the size constraint $|\calY_i| \leq R, \, \forall i \in [m]$, for some fixed integer $R$ satisfying $2 \leq R \leq K$. In light of Lemma~\ref{lem:fully_distinguishing}, we want to understand how the minimum number of rows, $m^*$, required to create a fully distinguishing classification matrix scales with the number of columns, $K$, and maximum row weight, $R$. 

\subsection{Bounds on the Number of Local Classifiers}\label{subsec:bounds}
We now give a lower bound of $m^* = \Omega(K^2/R^2)$ on the minimum number of rows in a fully distinguishing classification matrix. 

\begin{theorem}\label{thm:perfect_LB}
	For any integer $K \geq 2$, an $m \times K$ classification matrix that is fully distinguishing with maximum row weight $R \in\{2,3,\ldots, K\}$ must satisfy 
	\[m \geq \left\lceil\frac{K(K-1)}{R(R-1)} \right\rceil.\]
\end{theorem}
\begin{proof}
For a classification matrix to be fully distinguishing, it needs to satisfy ${K \choose 2} = K(K-1)/2$ constraints, namely that every pair of columns needs to share a 1 in some row. However, each row that we add to the matrix can satisfy at most ${R \choose 2} = R(R-1)/2$ such constraints, as the maximum weight of each row is $R$.
\end{proof}

The lower bound above is tight at the extreme values of $R$. When $R = 2$, perfect accuracy is achievable with $m = K(K-1)/2$ classifiers, i.e., using all ${K \choose 2}$ possible binary classifiers, then performing a majority vote to make predictions. This is the same as the well known one-vs.-one strategy which decomposes multiclass problems into pairwise binary problems. If $R = K$, then perfect accuracy is trivially achieved using just a single centralized classifier.

The following achievability result yields an upper bound of $m^* = O\Big(\frac{K^2}{R^2}\log K\Big)$ by using a probabilistic argument, where each classifier is drawn independently and uniformly at random from the set of size-$R$ classifiers. \footnote{Throughout, $\log$ denotes the natural (base $e$) logarithm.} Compared to the lower bound in Theorem \ref{thm:perfect_LB}, there is a $\log K$ gap in terms of scaling.  Additionally, we show that such a randomly selected set of classifiers satisfies the covering condition (i.e., yields a classification matrix that is fully distinguishing) with high probability. This result shows that in practice, if one wishes to label a dataset and selects sufficiently many labelers who specialize in random subsets of classes, then with high probability one obtains a reliably labeled dataset using the aforementioned decoding scheme.

\begin{theorem}\label{thm:perfect_UB}
	For all integers $K \geq 2$, there exists an $m \times K$ classification matrix with maximum row weight $R \in \{2,3,\ldots, K\}$ that is fully distinguishing with 
	\[m = \left\lceil\frac{K(K-1)}{R(R-1)}\cdot \log\Bigg(\frac{K(K-1)}{2}\Bigg) + 1\right\rceil. \]
	Moreover, an $m \times K$ classification matrix with each row selected independently and uniformly at random from the set of weight-$R$ binary vectors of length $K$ is fully distinguishing with probability at least $1-\delta$ if 
	\[m = \left\lceil \frac{K(K-1)}{R(R-1)} \cdot \log\Bigg(\frac{K(K-1)}{2\delta}\Bigg) \right\rceil.\]	
\end{theorem}

Resolving the gap between the bounds in Theorem~\ref{thm:perfect_LB} and Theorem~\ref{thm:perfect_UB} is an open problem. We conjecture that the upper bound in Theorem~\ref{thm:perfect_UB} gives sub-optimal scaling, and that the $\log K$ factor can be eliminated. In many combinatorial problems, such $\log$ factors reflect inefficiencies in the sampling procedure. Likewise, the first result in Theorem~\ref{thm:perfect_UB} is proved by (inefficiently) sampling classifiers uniformly at random with replacement. In the next section, we give greedy algorithms for selecting a minimal set of local classifiers which form a fully distinguishing classification matrix, i.e., satisfy the covering condition. As we will see, one of these algorithms produces a set of local classifiers which is guaranteed to be within a factor of $O(\log R)$ to optimality.

%% file: set_covering.tex
\section{Algorithms via Set Covering and Clique Covering}\label{sec:setcover}
The proof of Theorem \ref{thm:perfect_UB} relied on a random configuration of local classifiers to prove the existence of a configuration which satisfies the covering condition. In practice, however, it would be useful to have a more deterministic approach.

\subsection{Set Covering}
The problem of achieving perfect accuracy turns out to be a special case of the well-known \textit{set cover} problem from the combinatorics and theoretical computer science literatures, originally introduced and shown to be NP-complete in 1972 \cite{karp1972reducibility}. 
The set cover problem consists of 
\begin{enumerate}
\item  the \textit{universe}: a set of elements $\calU = \{1,2,\ldots, n\}$, and
\item a collection of sets $\calS$ whose union equals the universe, i.e., $\underset{S \in \calS}{\cup} S = \calU$. 
\end{enumerate}
The goal is to identify the smallest sub-collection of $\calS$ whose union equals the universe. Our classification problem can be reformulated in these terms by setting $\calU$ to be the set of all $n = {K \choose 2}$ pairwise connections between classifiers and setting $\calS$ to be the set of all ${K \choose R}$ possible sets of ${R \choose 2}$ pairwise connections that can be made in a single row of the classification matrix. Designing the classification matrix to have as few rows as possible while satisfying the conditions for perfect accuracy is equivalent to finding the smallest sub-collection of $\calS$ whose union equals $\calU$.

The following greedy algorithm \cite{chvatal1979greedy} gives a polynomial time (in $n \cdot |\calS|$) approximation of set covering. In each iteration, the algorithm simply chooses the set in $\calS$ that contains the largest number of yet-uncovered elements of $\calU$, and adds this set to the partial cover. This step is repeated until the union of the selected subsets covers $\calU$. For example, if $K=5$, $R=3$ in the classification problem, then the algorithm provides a set covering defined by $\calY_1 = \{1, 2, 3\}, \, \calY_2 = \{1, 4, 5\}, \, \calY_3 = \{2, 3, 4\}, \, \calY_4 = \{2, 3, 5\}$. This set of classifiers satisfies the covering condition from Corollary \ref{cor:covering}, and thus admits perfect accuracy under our model. 

This algorithm identifies a set covering that is at most $\tilde{H}(n)$ times as large as the optimal covering, where $\tilde{H}(n)$ is the $n^{\text{th}}$ harmonic number given by 	\[\tilde{H}(n) = \Sum{i=1}{n} \frac{1}{i} \leq \log n + 1.\]
In fact, if $|S| \leq \rho, \, \forall S \in \calS$, then the ratio is improved to $\tilde{H}(\rho)$ \cite{johnson1974approximation,  lovasz1975ratio}.
In the classification problem, we have $\rho = {R \choose 2}$, so the greedy algorithm finds a perfect-accuracy classifier configuration with at most 
$\tilde{H}({R \choose 2}) \approx 2\log R$
times as many classifiers as the optimal configuration. If $R$ is small (which is the expected regime in many applications), then the algorithm is nearly optimal. Numerically, one can also verify that the sizes of the covers produced by the algorithm nearly match the lower bound in Theorem \ref{thm:perfect_LB} in the small $R$ regime.

\subsection{Clique Covering}\label{subsec:clique_covering}
We note that the greedy algorithm runs in polynomial time \textit{in the parameters of the set cover problem}, $n$ and $|\calS|$, but that this translates to exponential time in $K$ and $R$. However, more efficient algorithms present themselves when we rephrase our problem in graph-theoretic terms. Consider an undirected graph $G = (V,E)$ with vertex set $V = \{1,\ldots, K\}$ and edge set $E$ where for every $k,k' \in [K], \, k \neq k'$, we have $(k,k') \in E$ if and only if $k,k' \in \calY_i$ for some $i \in [m]$. 
Thus, each classifier creates a \textit{clique} on the graph. Moreover, a configuration of classifiers can achieve perfect accuracy if and only if their induced graph, $G$, is the \textit{complete graph} \footnote{While the standard definition of a complete graph requires all pairs of vertices to be connected by a unique edge, we allow for redundant edges due to overlapping cliques.} on $K$ vertices. Our problem can be equivalently phrased as: What is the minimum number of cliques of size $R$ needed to cover the edges of the complete graph? 
This is a special case of the \textit{$k$-clique covering problem}, which was shown in \cite{holyer1981np} to be NP-complete in the general case. It is readily seen that clique covering is in turn a special case of set covering. This connection was studied in \cite{goldschmidt1996approximation}, and approximation algorithms which give better worst-case running times than greedy set covering -- at the expense of approximation ratio -- were provided. For simplicity, we use the standard greedy set covering algorithm in our experiments in Section \ref{sec:experiments}, but one can alternatively use the more efficient algorithms from \cite{goldschmidt1996approximation}.

\subsection{Variations on Set Covering for Dataset Construction and Multiclass Adaptation }
By setting appropriate inputs to a set cover algorithm, we can handle two types of situations that may arise in the context of dataset construction or multiclass classification. Consider a setting with $m$ classifiers or data labelers, denoted by $f_i: \calX \to \calY_i, \, i \in [m]$. First, suppose the cliques corresponding to the initial collection $\calY_1,\ldots,\calY_m$ fail to cover the edges of the complete graph on $K$ vertices, and that we are interested in determining the smallest number of \textit{additional} local classifiers (cliques) that need to be obtained to form a complete cover. In this case, we can set the universe $\calU$ to be the set of uncovered edges, and set $\calS$ to be the cliques that have not yet been used. By running a set cover algorithm on this problem instance, we obtain a minimal set of additional cliques needed to form a complete cover. In practice, the central orchestrator may leverage this information to seek out human labelers with the desired missing expertise, or to determine which additional local classifiers need to be trained to adapt an existing set of classifiers to a larger multiclass objective. 

Second, suppose there are \textit{more} local classifiers than are needed to cover the complete graph. This is likely to be the case in large-scale classification or crowdsourcing settings with numerous participating entities. To minimize costs, the central orchestrator may be interested in selecting a minimal subset of these classifiers which still covers the graph, and asking only the corresponding ones to send their labels. Here, we can set $\calS$ equal to the cliques available in the pool of classifiers, and keep $\calU$ as the set of edges in the complete graph on $K$ vertices. The set cover algorithm will then return a (close to) minimal subset of $\calS$ which still covers the complete graph. 

%% file: statistical.tex
\section{Statistical Setting}\label{sec:statistical}
The perfect accuracy setting from Section~\ref{sec:perfect} was combinatorial in nature and led to worst-case bounds on the number of local classifiers needed to exactly recover the true class. We now investigate the scaling of $m^*$ in a more average-case sense, 
as described in Section~\ref{subsec:statistical}. In practice, classifiers are likely to lie somewhere between these two noise models. In the statistical setting, we are particularly interested in whether one can still achieve a high classification accuracy with fewer classifiers than in the entire set cover. This problem is relevant, for example, when it is  difficult to ensure that all of the set-covering classifiers are available, or when one is willing to sacrifice a small amount of accuracy for the sake of reducing data labeling expenditures. 

\subsection{Lower Bound}
We first give an information-theoretic lower bound on $m^*$ using Fano's inequality. Our proof will rely on the following lemma, which gives an expression for the conditional entropy of the classifier outputs, $Y$, given the true class, $Z$. 

\begin{lemma}\label{lemma:entropy}
	Under the assumptions of Section~\ref{subsec:statistical}, the conditional entropy of $Y = (Y_1,\ldots, Y_m)$ given $Z$ is \[H(Y\,|\,Z) = m\cdot \frac{(K-R)}{K}\cdot \log R. \]
\end{lemma}

Next, we show that $m^* = \Omega((K\log K)/(R\log R))$, which hints that the more benign conditions of this statistical setting may lead to roughly a factor of $\frac{K}{R}$ reduction (up to $\log$ factors) in the number of classifiers that are required compared to the perfect accuracy setting. 

\begin{theorem}\label{thm:stat_LB}
	Any $K$-class classification scheme using $m$ smaller classifiers of size $R$ that achieves $P_e \leq \epsilon$ under the assumptions of Section~\ref{subsec:statistical} must satisfy \[m \geq \left\lceil \frac{K}{R}\cdot \Bigg(\frac{(1-\epsilon)\log K - \log(2)}{\log R} \Bigg) \right\rceil.\]
\end{theorem}

\subsection{Upper Bound}
To prove an upper bound on the number of classifiers required to achieve an $\epsilon$-probability of error, we consider the same probabilistic construction as in Theorem~\ref{thm:perfect_UB}, coupled with a maximum likelihood (ML) decoding strategy. 
For a particular output vector $y=(y_1,\ldots,y_m)$, the decoder predicts the class according to the decision rule  
\[g(y) = \underset{k\in[K]}{\text{argmax}} \; \calL(y;k) \]
where 
\[\calL(y;k) = \Prod{i=1}{m}\PP(Y_i=y_i \, | \, Z=k)\]
and with ties broken arbitrarily.

\textbf{Remark:}  The above ML decoding scheme potentially generalizes to the setting in which we have local classifiers that are \textit{not perfect} but whose respective accuracies (or reliabilities, in the case of human labelers) can be estimated. For example, one could use an expectation-maximization (EM) \cite{dempster1977maximum} based algorithm to estimate these quantities, similar to those proposed in several papers on crowdsourcing \cite{dawid1979maximum, jin2003learning, raykar2010learning, sheng2008get}. 

\begin{theorem}\label{thm:stat_UB}
	Under the assumptions of Section~\ref{subsec:statistical}, the previously described classifier construction and decoding rule achieve a probability of error bounded by an arbitrary $\epsilon \in (0,1)$ using \[m = \left\lceil \frac{K(K-1)}{(K-R)(R-1)}\cdot \log\Big(\frac{K}{\epsilon}\Big) \right\rceil\] classifiers of size $R$. 
\end{theorem}

For fixed $\epsilon$, the above result gives $m^* = O((K/R)\log K)$ when $R$ is sufficiently smaller than $K$, which is only a factor of $\log R$ larger than our bound in Theorem~\ref{thm:stat_LB} in terms of scaling. When $R = O(1)$, then the upper and lower bounds meet, yielding $m^* = \Theta(K\log K)$. On the other hand, when $R = K^\alpha$ for some $\alpha \in (0,1)$, then the lower bound scales as $\Omega(K)$, whereas the upper bound scales as $O(K\log K)$. 

When combined, Theorems~\ref{thm:stat_LB} and \ref{thm:stat_UB} reveal that relaxing the criteria for perfect accuracy yields a reduction in the minimum number of required classifiers from roughly $\Theta(K^2/R^2)$ to $\Theta(K/R)$. Note that $\ceil{K/R}$ is the minimum number of size-$R$ cliques needed to cover $K$ vertices. Therefore, under the graph-theoretic interpretation given in Section~\ref{subsec:clique_covering}, the problem now roughly reduces to covering the vertices rather than the edges of the complete graph on $K$ vertices.

%% file: experiments.tex
\section{Experiments}\label{sec:experiments}
We present experimental results on the performance of our aggregation scheme applied to classifiers trained on subsets of a global dataset. For different classifier sizes $R$, we used the greedy set cover algorithm to design configurations of smaller classifiers. For example, for $K=10$ and $R=4$, we trained 9 smaller classifiers, each given access to all training examples corresponding to 4 classes. We used the decoding scheme given in Equation (\ref{eqn:decoder}), Section \ref{subsec:decoder}. We examined the performance of this scheme on the MNIST \cite{lecun1998gradient} and CIFAR-10 \cite{krizhevsky2009learning} datasets, comparing the resulting classification accuracy for $R \in \{4,6,8\}$ to that of one-versus-one ($R=2$) and fully centralized classification ($R=10$). All implementations were done with Keras \cite{chollet2015keras}, and our experiments on \textit{FederatedAveraging} \cite{mcmahan2017communication} additionally utilized the TensorFlow Federated framework \footnote{\url{https://www.tensorflow.org/federated}.}. 
All hyperparameters were set to their default values. 

\subsection{Set Covering-Based Classification}

\subsubsection{MNIST}\label{subsec:mnist_experiment}
For the MNIST handwritten digit dataset, we used a convolutional neural network (CNN) architecture \footnote{\url{https://keras.io/examples/mnist_cnn/}.}. All classifiers were trained with the same architecture, except with possibly different dimensions in the final softmax layer. The batch size was set to 128, and training was done over 12 epochs per classifier. Table~\ref{table:mnist} shows the resulting training and testing accuracies. We see that our aggregation scheme performs nearly as well as the centralized classifier. 

\subsubsection{CIFAR-10}\label{subsec:cifar_experiment}
For the CIFAR-10 dataset, we used the ResNet20 v1 \footnote{\url{https://keras.io/examples/cifar10_resnet/}.} \cite{he2016deep} architecture for each classifier, with a batch size of 32 and 200 epochs. Table~\ref{table:cifar} 
shows the final training and testing accuracies, again demonstrating the favorable performance of our scheme.

\begin{table}[t]
	\centering 
	\caption{MNIST accuracies obtained by aggregating local classifiers designed by greedy set covering. }	
	\label{table:mnist}
	\begin{scriptsize}	
	\begin{tabular}{ lccccr }
		\toprule
		$R = $& $2$ & $4$ & $6$ & $8$ & $10$ \\
		\bottomrule 
		\toprule 
		Train (\%) & 99.71 & { 99.82} & { 99.74} & 99.78 & 99.12 \\
		\midrule 
		Test (\%) & 98.98 & { 99.03} & 98.97 & { 99.07} & 99.26 \\
		\bottomrule 
	\end{tabular}
	\end{scriptsize}
\end{table}

\begin{table}[t]
	\centering 
	\caption{CIFAR-10 accuracies obtained by aggregating local classifiers designed by greedy set covering. }	
	\label{table:cifar}
	\begin{scriptsize}
	\begin{tabular}{ lccccr }
		\toprule 
		$R=$ & $2$ & $4$ & $6$ & $8$ & $10$\\
		\bottomrule 
		\toprule  
		Train (\%) & 99.61 & {99.49}  & { 99.55}  &  { 99.16} & 98.48  \\
		\midrule 
		Test (\%) & 88.36 & { 90.82}  & { 91.95}  &  { 91.99} & 91.98 \\
		\bottomrule 
	\end{tabular}
	\end{scriptsize}
\end{table}

\subsection{Comparison to FederatedAveraging}\label{sec:fedavg}

\begin{figure*}[t]
	\centering
	\subfigure[2 clients per round]{\includegraphics[width=0.24\textwidth]{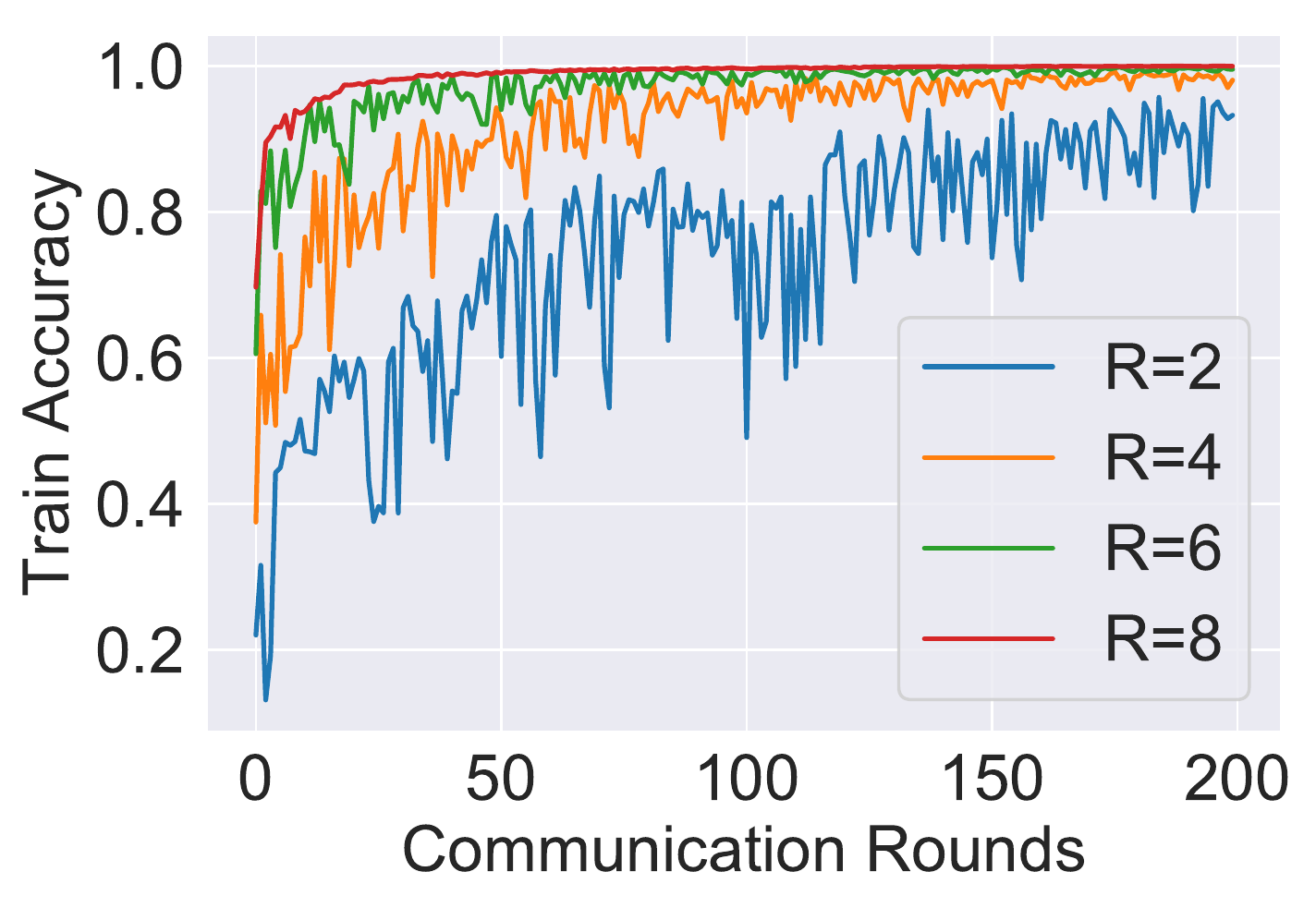}}
	\subfigure[2 clients per round]{\includegraphics[width=0.24\textwidth]{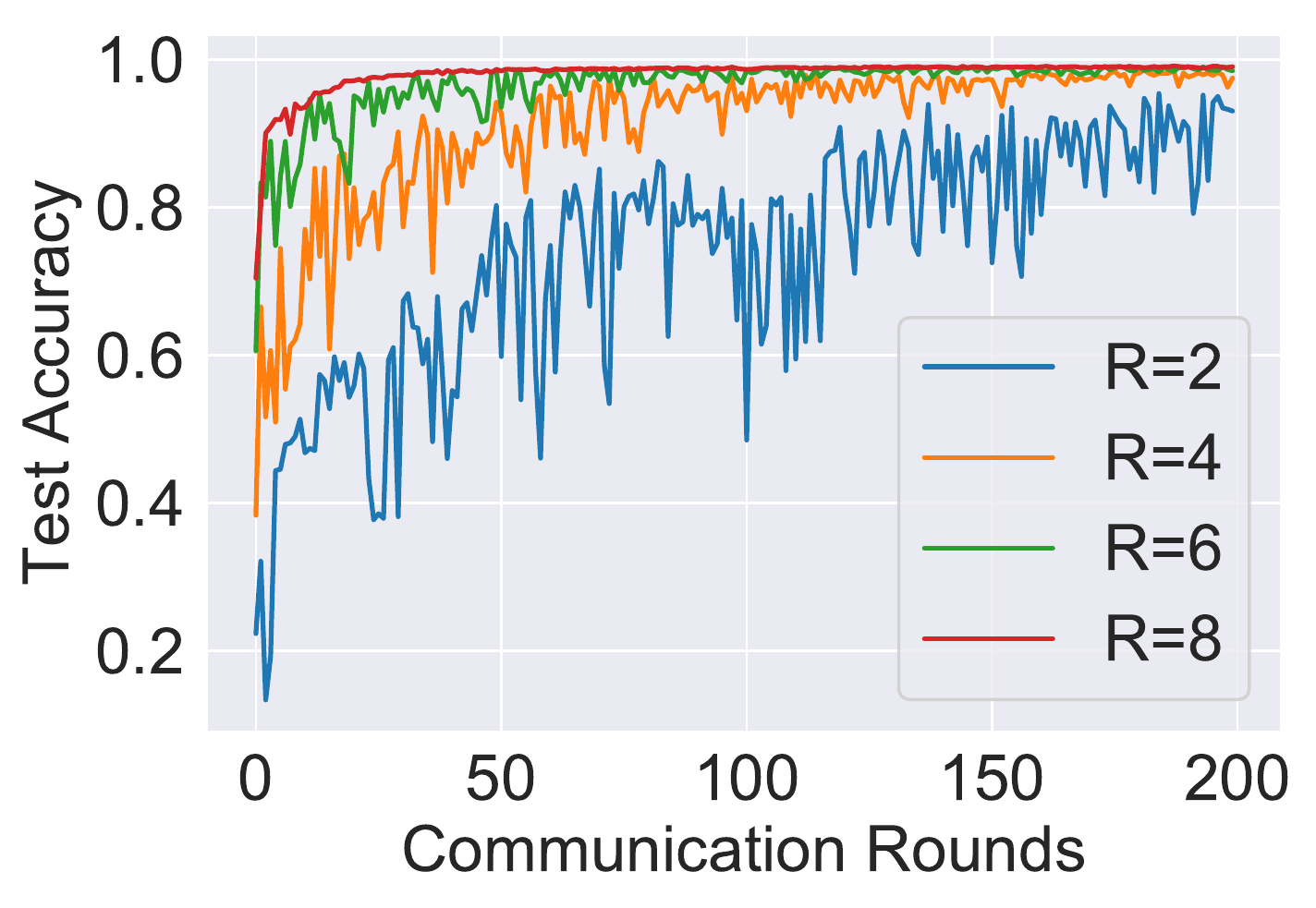}}
	\subfigure[Full client participation]{\includegraphics[width=0.24\textwidth]{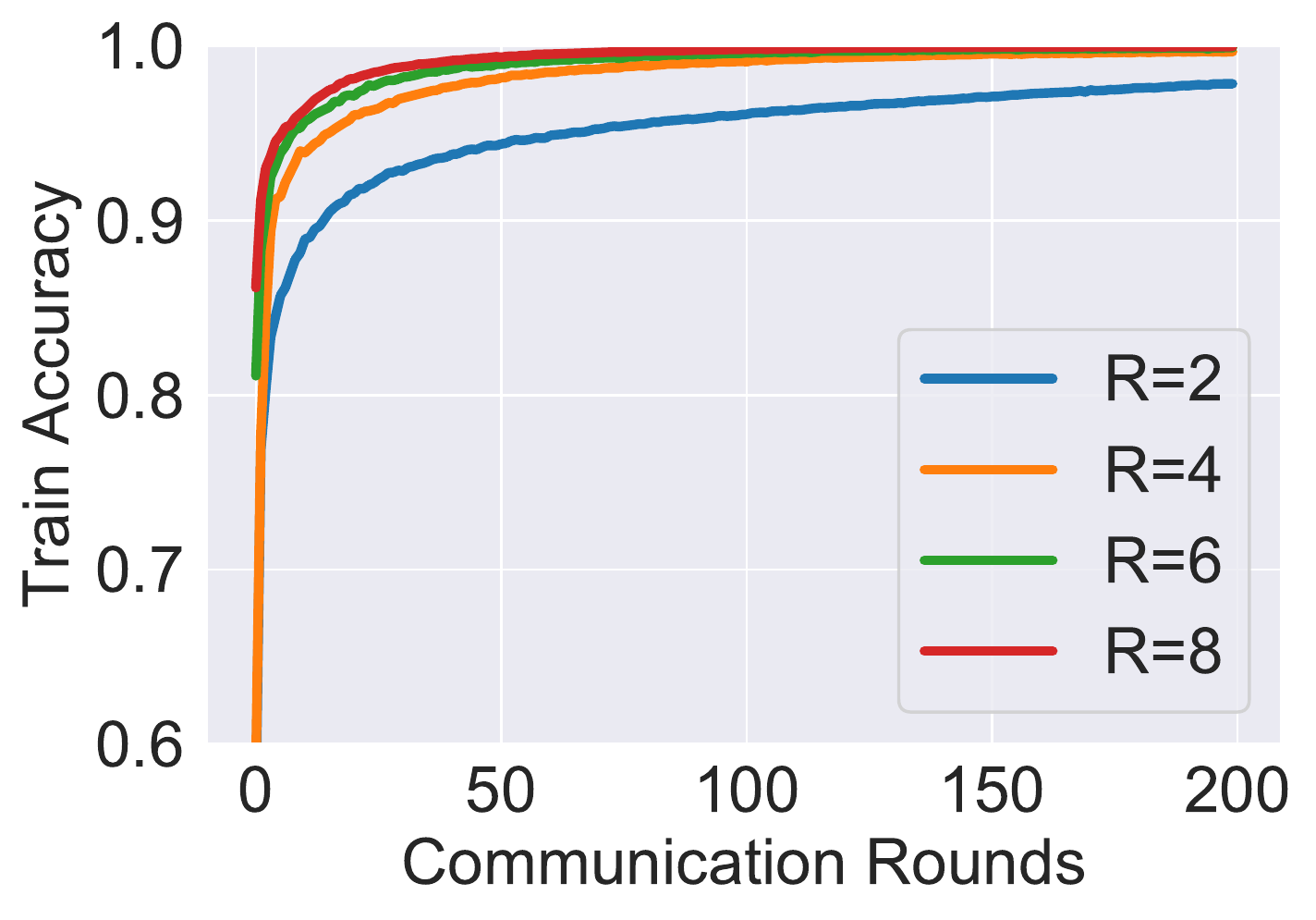}}
	\subfigure[Full client participation]{\includegraphics[width=0.24\textwidth]{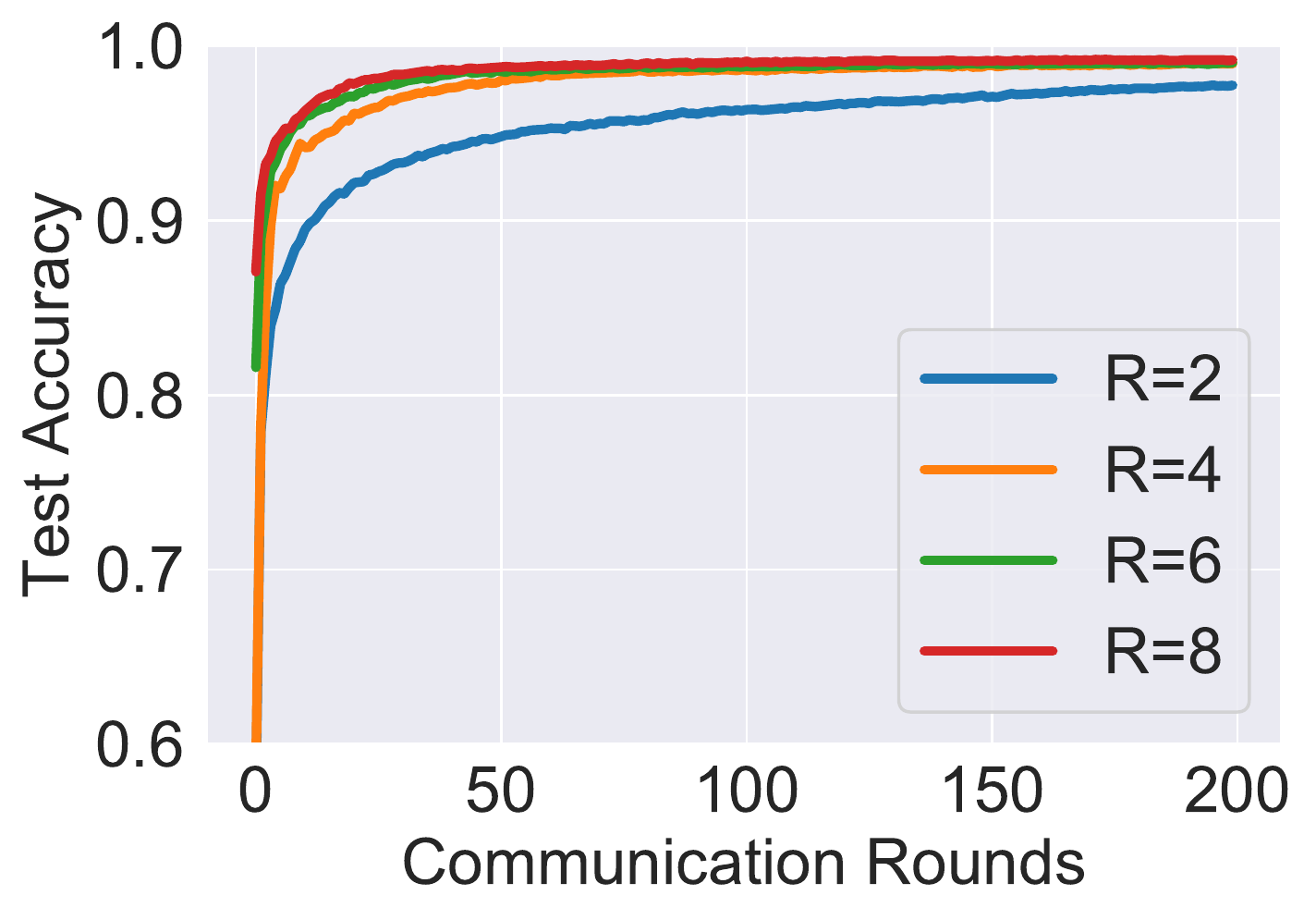}}
	\caption{Learning curves for \textit{FedAvg} on non-i.i.d. MNIST data partitions. (a) and (b): 2 clients selected uniformly at random to participate in each communication round. (c) and (d): All clients participate in each round.}
	\label{fig:fed_mnist}
\end{figure*}

\begin{figure*}[t]
	\centering
	\subfigure[2 clients per round]{\includegraphics[width=0.24\textwidth]{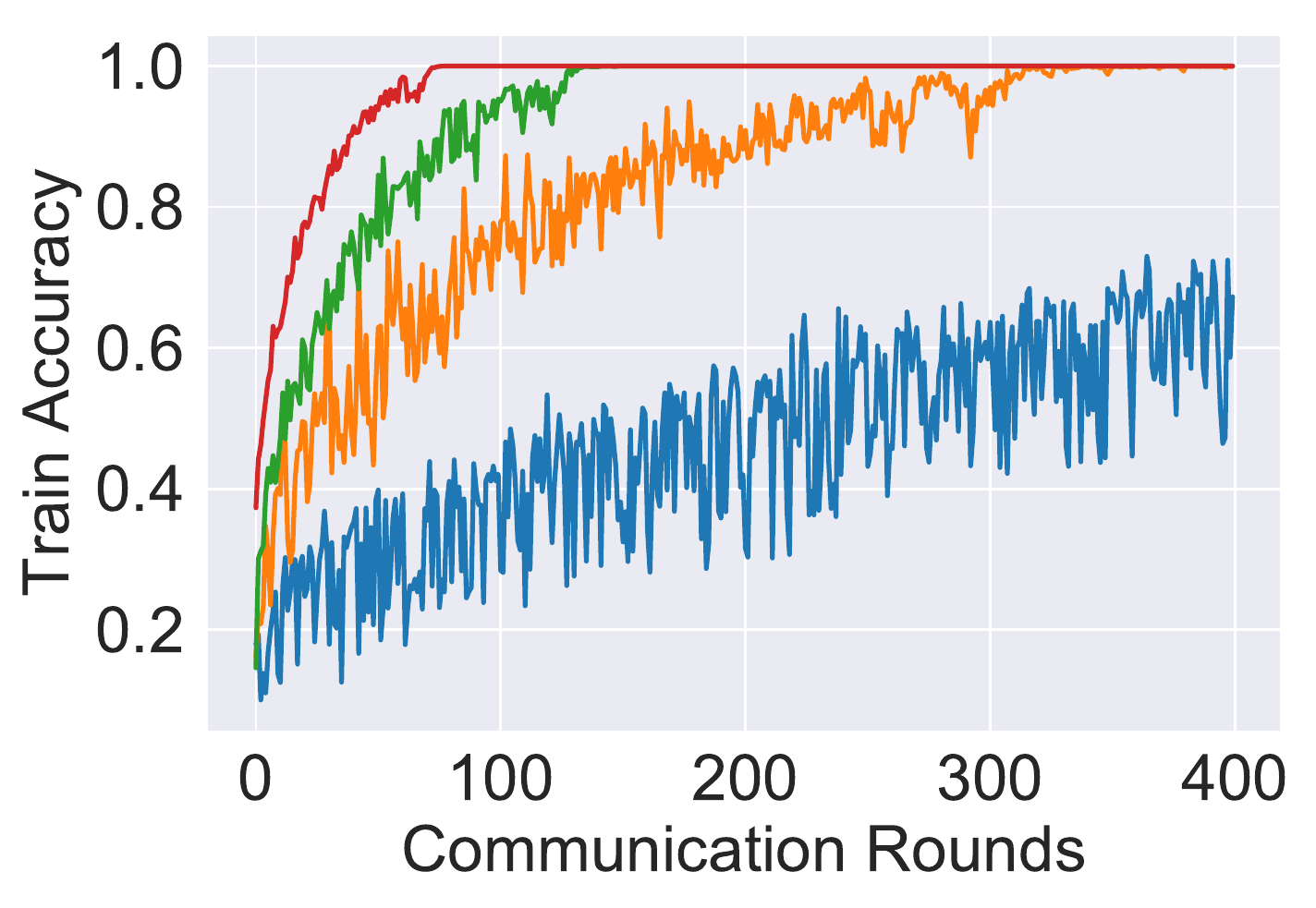}}
	\subfigure[2 clients per round]{\includegraphics[width=0.24\textwidth]{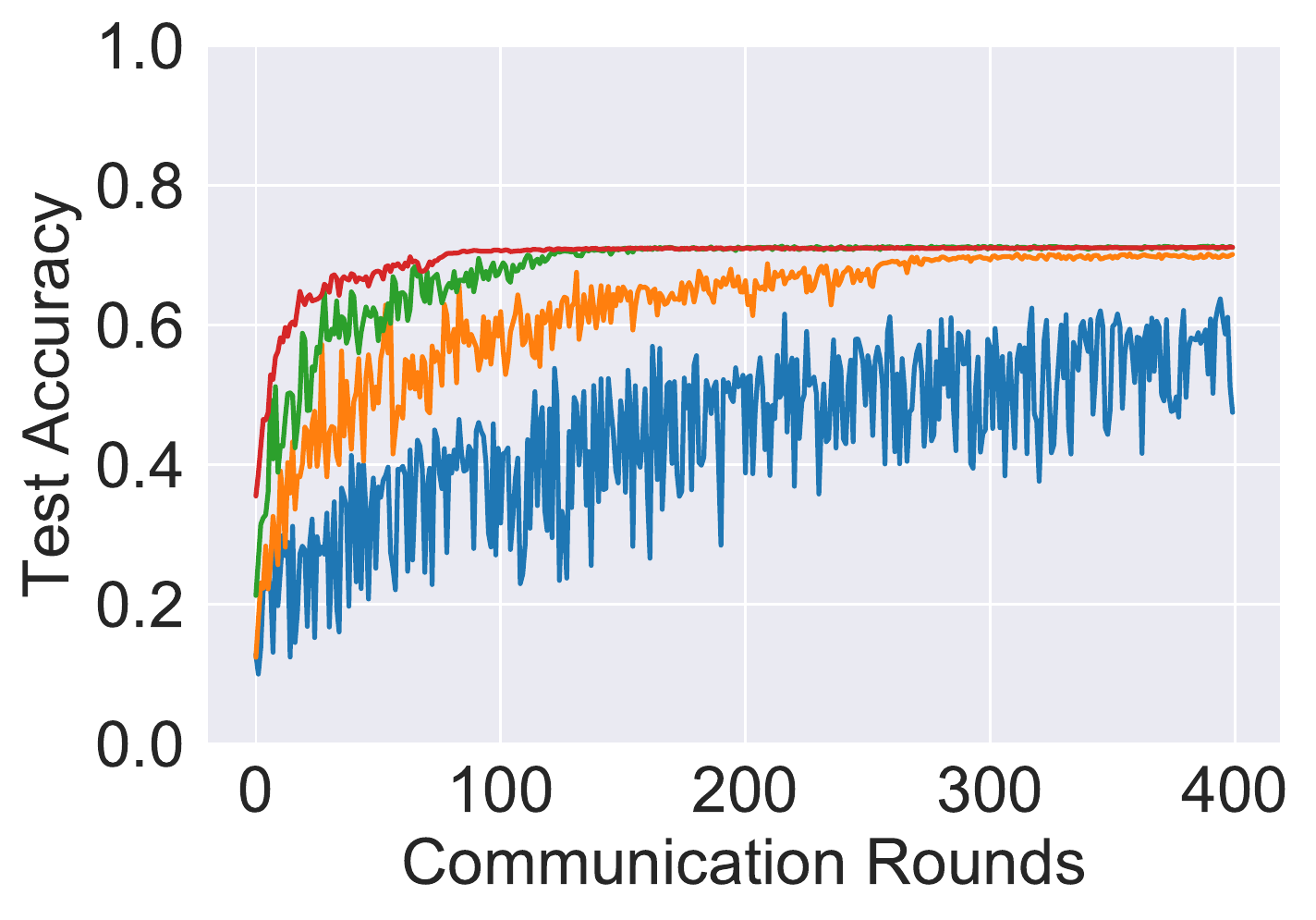}}
	\subfigure[Full client participation]{\includegraphics[width=0.24\textwidth]{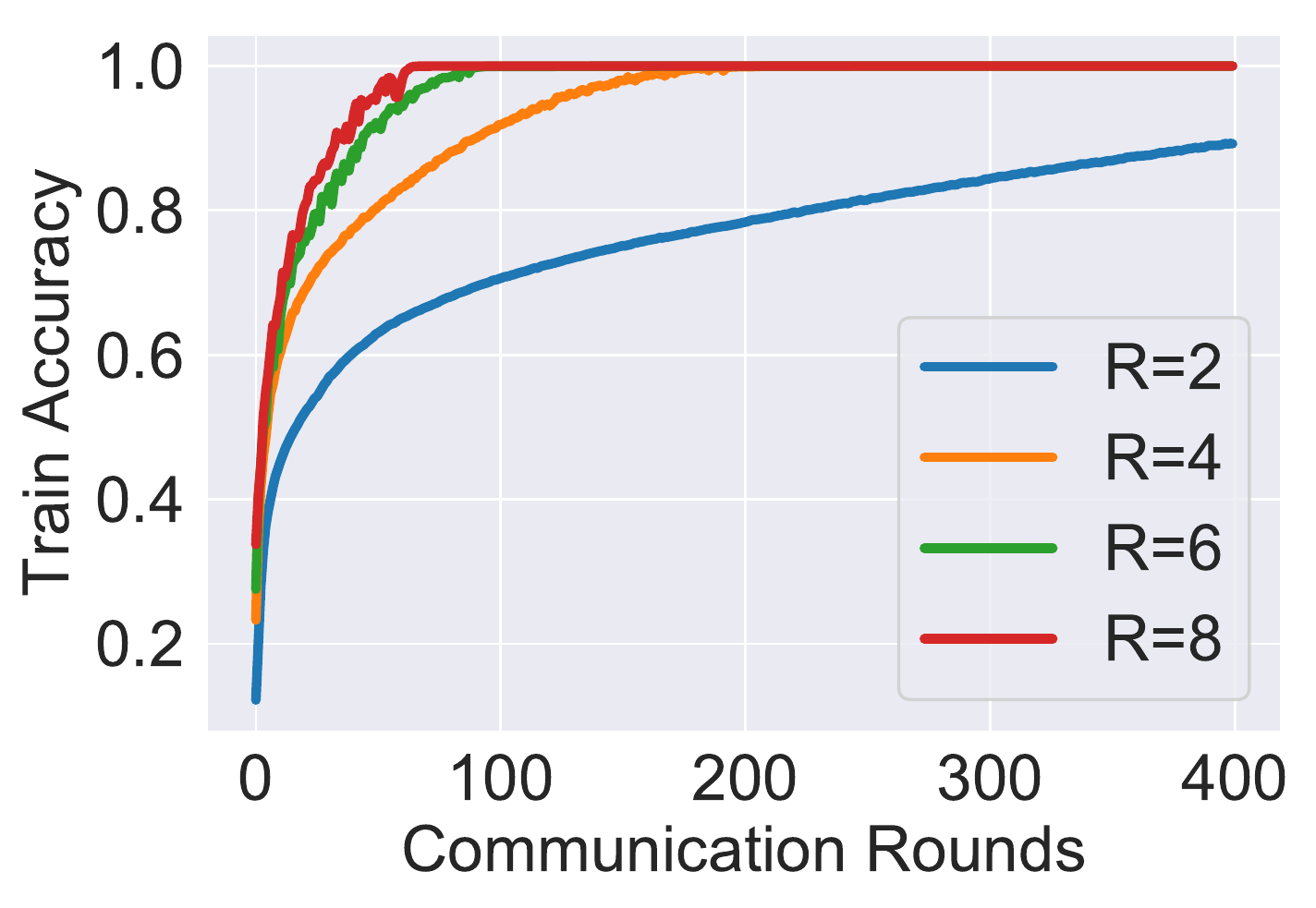}}
	\subfigure[Full client participation]{\includegraphics[width=0.24\textwidth]{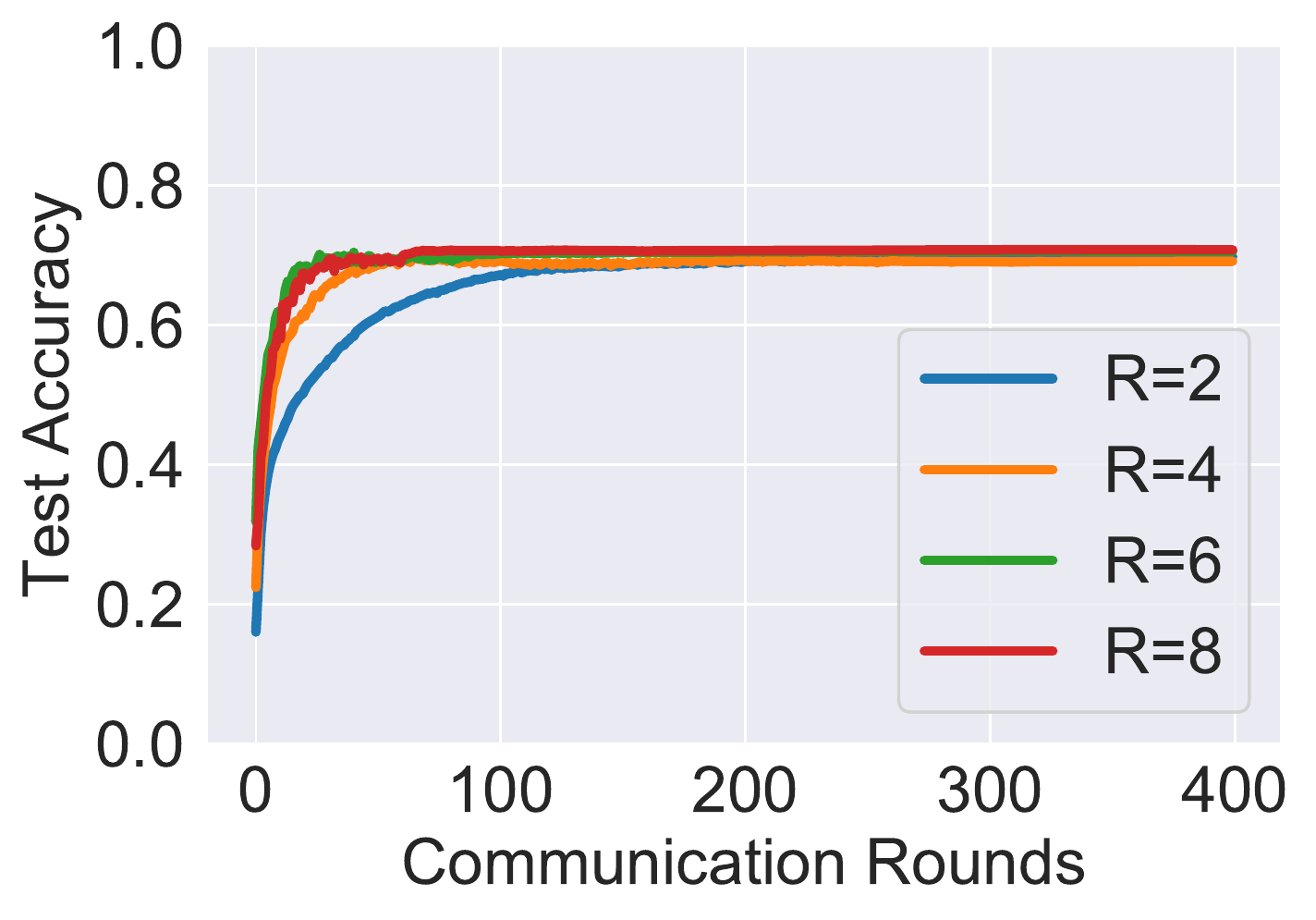}}
	\caption{Learning curves for \textit{FedAvg} on non-i.i.d. CIFAR-10 data partitions. (a) and (b): 2 clients selected uniformly at random to participate in each communication round. (c) and (d): All clients participate in each round.}
	\label{fig:fed_cifar}
\end{figure*}

Our approach is potentially relevant to the new area of \textit{cross-silo} federated learning \cite{kairouz2019advances} where the goal is to learn a global model from decentralized data located at different silos, such as hospitals. We further elaborate on this connection in Section~\ref{sec:discussion}. The following experiments explore whether our approach of aggregating local models that may have been produced by separate silos can outperform standard federated learning procedures. 

We implemented the standard \textit{FederatedAveraging} (or \textit{FedAvg}) algorithm on non-i.i.d. partitions of the MNIST and CIFAR-10 datasets. For each value of $R \in \{2,4,6,8\}$, we partitioned the training data according to the corresponding set cover produced by the greedy algorithm. For example, $R=2$ corresponds to 45 total clients, each with the training data for two classes. In each communication round, each client performed 1 epoch over their own data.

\subsubsection{MNIST} 
The batch size was set to 128 and the number of communication rounds was 200. The model architecture was the same as in Section~\ref{subsec:mnist_experiment}. We first examined the performance of \textit{FedAvg} when two clients are selected uniformly at random to participate in each communication round. The resulting train and test accuracy curves are plotted in Figure~\ref{fig:fed_mnist} (a) and (b). Evidently, the learning process is slower and more volatile for smaller values of $R$ (i.e., more clients). After 100 communication rounds, for instance, the test accuracy for $R=2$ was 48.5\%, and the final test accuracy after 200 rounds for $R=2$ was $93.1\%$. In Figure~\ref{fig:fed_mnist} (c) and (d), we see that the learning curves are much smoother when all clients participate in each communication round. 

\subsubsection{CIFAR-10} The batch size was 32 and the number of communication rounds was 400. We used the CNN architecture provided in the TensorFlow tutorial \footnote{\url{https://www.tensorflow.org/tutorials/images/cnn}. 
	We used this model for simplicity, and acknowledge that it is difficult to draw comparisons to our results in Section~\ref{subsec:cifar_experiment}, where we used a ResNet20 architecture. However, we suspect that \textit{FedAvg} may perform even worse with larger models. }. Figure~\ref{fig:fed_cifar} shows the analogous curves to those in Figure~\ref{fig:fed_mnist}. 
After 200 communication rounds, the train accuracy for $R=2$ was 31.52\% in the case of 2 clients per round, and 78.37\% with full client participation. The final train accuracy after 400 rounds with full participation was 89.24\%. 
With 2 clients participating per round, the test accuracy for $R=2$ after 400 communication rounds was $47.5\%$ (compared to $67.25\%$ train accuracy), and the maximum test accuracy over all values of $R$ was $71.1\%$ (compared to $100\%$ train accuracy). Full client participation improved the speed of convergence, but did not improve the resulting accuracy: the maximum final test accuracy over all values of $R$ was $70.8\%$ in this case. 

Our results suggest that \textit{FedAvg} is highly sensitive to both the heterogeneity of the data and the amount of client participation per round, whereas our aggregation approach seems to enjoy greater robustness to heterogeneity and does not rely on the consistent availability of clients over such long time horizons.

\subsection{Robustness to Missing Classifiers}\label{supp:missing}

\begin{table*}[t]
	\centering 
	\caption{Average CIFAR-10 accuracies under classifier removals from the set cover. Standard deviation is denoted by $\sigma$.}	
	\label{table:cifar_missing}
	\begin{scriptsize}
		\begin{tabular}{ lcccc }
			\toprule 
			& & Number of removed classifiers = 1 & & \\
			\midrule 
			$R=$ & $2$ & $4$ & $6$ & $8$\\
			\bottomrule 
			\toprule  
			Train (\%) & \textbf{98.63} ($\sigma = 1.62\%$) & \textbf{94.14} ($\sigma = 3.06\%$) & \textbf{96.30} ($\sigma = 2.86\%$)  & \textbf{95.54} ($\sigma = 4.03\%$)   \\
			\midrule 
			Test (\%) & \textbf{87.51} ($\sigma = 1.48\%$) & \textbf{86.11} ($\sigma = 2.80\%$) & \textbf{88.78} ($\sigma = 2.64\%$)  & \textbf{88.28} ($\sigma = 3.80\%$)   \\
			\bottomrule 
		\end{tabular}
		\vskip 0.15in 
		\begin{tabular}{ lcccc }
			\toprule 
			& & Number of removed classifiers = 2 & & \\
			\midrule
			$R=$ & $2$ & $4$ & $6$ & $8$\\
			\bottomrule 
			\toprule  
			Train (\%) & \textbf{97.58} ($\sigma = 2.32\%$) & \textbf{86.54} ($\sigma = 5.20\%$) & \textbf{88.38} ($\sigma = 4.77\%$)  & \textbf{79.26} ($\sigma = 0.25\%$)   \\
			\midrule 
			Test (\%) & \textbf{86.62} ($\sigma = 2.11\%$) & \textbf{79.28} ($\sigma = 4.80\%$) & \textbf{81.41} ($\sigma = 4.42\%$)  & \textbf{73.73} ($\sigma = 0.83\%$)   \\
			\bottomrule 
		\end{tabular}
	\end{scriptsize}
\end{table*}

We tested the robustness of our scheme under incomplete sets of classifiers using CIFAR-10. For each value $R \in \{2,4,6,8\}$, we initially generated a set of classifiers of size $R$ using the greedy set cover algorithm described in Section \ref{sec:setcover}. These were the same classifiers used in Section \ref{subsec:cifar_experiment}. We investigated the effect on the train and test accuracies of removing \textit{(a)} a single classifier, and \textit{(b)} two classifiers from the set cover. For \textit{(a)}, we determined the resulting accuracies after removing each classifier, and then computed the average over all such removals. For \textit{(b)}, we considered all ${m^* \choose 2}$ removals of two classifiers, where $m^*$ is the size of the set cover. 

The average accuracies for each value of $R$ and their corresponding standard deviations are shown in Table \ref{table:cifar_missing}. Surprisingly, the performance under single classifier removals remained quite similar to the values in Table \ref{table:cifar}. The values for $R=2$ remained closest to those from the complete set cover, and also had the smallest variance among all values of $R$. This is likely due to the fact that smaller values of $R$ correspond to more classifiers in the set cover, each responsible for fewer classes. Hence, the removal of a single classifier when $R$ is small is likely to have a less detrimental effect on performance.

\subsection{Random Classifier Configurations}
A related question to that addressed in Section~\ref{supp:missing} is whether a random set of classifiers performs well. In practice, it may not be possible to control the distribution of data across separate entities, and randomly chosen classifiers can serve as a model for such situations. For each $R \in \{4,6,8\}$, we sampled uniformly without replacement $m$ times from the set of ${K \choose R}$ possible classifiers of size $R$, for $m \in \{m^*, (m^*+1), (m^*+2), (m^*+3)\}$, where $m^*$ is the number of classifiers specified by the greedy set cover algorithm. That is, we generated slightly more classifiers than those in the set cover. Note that a random set of classifiers, even of cardinality larger than $m^*$, may fail to comprise a complete set cover. 

Each randomly chosen classifier was trained on CIFAR-10 with the same architecture and hyperparameters as in Section \ref{subsec:cifar_experiment}. 
Table \ref{table:cifar_random} gives the resulting training and testing accuracies of our scheme (the values shown are averages over two trials). Overall, we observe that a set of exactly $m^*$ random classifiers tends to yield (except in the case of $R=8$) lower accuracies than those attained by the complete set cover. For example, with $R=4$ and $m = m^* = 9$, the random construction achieved train and test accuracies of (respectively) $91.66\%$ and $84.33\%$, whereas the complete set cover achieved respective accuracies of $99.49\%$ and $90.82\%$. However, with the addition of a few more classifiers, the accuracy can increase dramatically and can even surpass the set cover accuracy. For instance, with $R=8$ and $m = 6$, the train and test accuracies of the random construction were $99.75\%$ and $93.37\%$, respectively, in contrast to $99.16\%$ and $91.99\%$ from the full set cover. We speculate that the random construction's ability to outperform the set cover is due to the introduction of redundant connections between classes, which may constitute a form of natural error correction. In summary, these results suggest that our scheme is robust not only to uncovered connections between classes, but also to arbitrary distributions of the data across clients.

\begin{table*}[t]
	\centering 
	\caption{CIFAR-10 accuracies with random classifier configurations. The number of randomly chosen classifiers is denoted by $m$. }
	\label{table:cifar_random}
	\begin{scriptsize}
	\begin{tabular}{lll}
		
		\begin{tabular}{ lcccc }
			\toprule 
			& & $\mathbf{R = 4}$ & & \\
			\midrule  
			$m=$ & $9$ & $10$ & $11$ & $12$ \\
			\bottomrule 
			\toprule  
			Train (\%) &   91.66  & 92.15   & 92.44 &  94.28   \\
			\midrule 
			Test (\%)  &  84.33  & 84.98  & 85.47  & 87.22  \\
			\bottomrule 
		\end{tabular}
		&\hspace{0.2in}
		\begin{tabular}{ lcccc }
			\toprule 
			& & $\mathbf{R = 6}$ & & \\
			\midrule  
			$m=$ & $5$ & $6$ & $7$ & $8$ \\
			\bottomrule 
			\toprule  
			Train (\%) &  93.97   & 98.78  & 98.87 & 99.74    \\
			\midrule 
			Test (\%)  &  86.70 &  91.60 & 91.82 & 92.76  \\
			\bottomrule 
		\end{tabular}
	\end{tabular}
	\vskip 0.1in 
	\begin{tabular}{ lcccc }
		\toprule 
		& & $\mathbf{R = 8}$ & & \\
		\midrule  
		$m=$ & $3$ & $4$ & $5$ & $6$ \\
		\bottomrule 
		\toprule  
		Train (\%) & 95.71 & 99.48  & 99.75 & 99.75    \\
		\midrule 
		Test (\%)  & 88.88 & 92.60 & 93.24  & 93.37 \\
		\bottomrule 
	\end{tabular}
	\end{scriptsize}
\end{table*}

\subsection{Performance with Less Data and Less Training Time}

\begin{table}
	\centering 
	\caption{CIFAR-10 accuracies with less training data and training time per classifier.}
	\label{table:cifar_lessdata}
	\begin{scriptsize}
		\begin{tabular}{ lcccc }
			\toprule 
			$R=$ & $2$ & $4$ & $6$ & $8$\\
			\bottomrule 
			\toprule  
			Train (\%) & 71.75 & 74.55  & 78.32  & 78.23   \\
			\midrule 
			Test (\%) & 68.56  & 71.37 & 75.64 & 75.60  \\
			\bottomrule 
		\end{tabular}
	\end{scriptsize}
\end{table}

In practice, it may also be difficult to ensure that each local classifier has access to sufficiently many training samples. To study the performance of our scheme under such scenarios, we reduced the number of CIRAR-10 training images per class to 500, compared to 5,000 images per class in the entire training set. We also trained each classifier for 100 epochs, compared to 200 epochs in other experiments in this paper. For each $R \in \{2,4,6,8\}$, we trained the classifiers specified by the greedy set cover algorithm. All other details match Section \ref{subsec:cifar_experiment}. As shown in Table \ref{table:cifar_lessdata}, there is a noticeable degradation in both training and testing accuracy compared to Table \ref{table:cifar}. 
However, even with only a fraction of the original training data and significantly fewer training epochs than usual -- two relevant characteristics for real-world settings -- our aggregation approach exhibits acceptable performance.

%% file: discussion.tex
\section{Discussion and Future Work}\label{sec:discussion}
First, we discuss potential connections of our work to \textit{federated learning} (FL) \cite{mcmahan2017communication, konevcny2016federated, konevcny2016federatedopt}. In recent years, FL has emerged as a promising decentralized learning paradigm, where a central server trains a shared machine learning model using structured updates from multiple clients. Despite the recent success of FL, the presence of \textit{statistically heterogeneous} or \textit{non-i.i.d.} data is known, both theoretically and experimentally, to be detrimental to the convergence of existing federated algorithms \cite{zhao2018federated, li2020on, mcmahan2017communication,eichner2019semi, smith2017federated, li2020federated, wang2019adaptive, sattler2019robust}. In practice, decentralized data located at different edge devices exhibit very different statistical characteristics, and the sources of such heterogeneity are often quite natural. For instance, users living in different geographical regions are likely to have different types of photos on their mobile phones or language patterns manifest in text messages.

The approach we propose in this paper can potentially mitigate the effects of statistical heterogeneity in the emerging sub-area known as \textit{cross-silo} federated learning \cite{kairouz2019advances}. Naturally, different organizations or silos, such as hospitals, have access to heterogeneous types of data and may seek to form an accurate global prediction model by fusing their local models. Moreover -- consistent with the mathematical model proposed in our work -- it is natural to expect that each organization has access to sufficient training data to learn an accurate local model. Treating the local classifiers as black boxes also allows the silos to use different learning algorithms and architectures (e.g., decision trees, support-vector machines, or neural networks). This paper can be viewed as a contribution to the rigorous study of this growing field, and is similar in spirit to previously proposed solutions to the FL heterogeneity problem which suggest to maintain separate, personalized models for each client or each cluster of similar clients \cite{smith2017federated, fallah2020personalized, mansour2020approaches, eichner2019semi}. 

Future work includes generalizing our results to other probabilistic  models, as well as to the multi-label classification setting in which each sample may belong to more than one category. In this case, we suspect that the corresponding algorithmic approach will be to solve a set cover problem on a hypergraph, rather than on a standard graph as studied in the multiclass setting. We also seek to identify order-optimal deterministic constructions for both the perfect accuracy and statistical settings. 

 Other future directions include using tools from coding theory to increase the robustness of the classification schemes to poorly trained or byzantine classifiers, and performing experiments on large-scale datasets such as ImageNet.  Finally, if the learning algorithms used to produce the local classifiers are known in advance, how might we leverage this information to design better aggregation schemes?

%% file: related_work.tex
\section{Related Work}\label{sec:related}
\subsection{Crowdsourcing}
The problem of inferring the true underlying class of a sample from noisy observations is a longstanding one \cite{dawid1979maximum}, and has more recently re-emerged as an active area of research in the context of crowdsourcing. Many approaches (e.g., \cite{dawid1979maximum, jin2003learning, raykar2010learning, sheng2008get})  are based on the iterative expectation-maximization (EM) algorithm \cite{dempster1977maximum}. However, these approaches are based on heuristics which have been primarily evaluated through numerical experiments. 

A recent line of work \cite{karger2014budget, NIPS2011_c667d53a, karger2013efficient, ghosh2011moderates, vempaty2014reliable} takes a more rigorous approach to the design of crowdsourcing systems. With the exception of \cite{karger2013efficient}, these works focus on binary classification tasks, rather than the multiclass setting we consider. Our work is still distinct from \cite{karger2013efficient} in a couple of key ways. First, \cite{karger2013efficient} adopts a probabilistic model in which the confusion matrices of the workers are assumed to be drawn from a common distribution. In our setting, each worker has its own distribution which depends on its domain expertise. We also restrict the support of each worker's labels to be the set of classes for which it is an expert. Second, a large focus of our work is an adversarial setting rather than a probabilistic one. It is interesting to note, however, that the upper bound of $O((K/q)\log(K/\epsilon))$ from \cite{karger2013efficient} on the number of workers per task their scheme needs to accurately label a sample -- where $q$ is a ``crowd-quality'' parameter defined in their probabilistic model -- resembles the scaling we obtained in our statistical setting.

\subsection{Dataset Construction and Self-Training} 
In contrast to the crowdsourcing literature, our work also explores potential applications to automated data labeling. We have demonstrated how one can use a set of local classifiers to generate labels for an even larger scale multiclass dataset. As discussed in the previous section, this may be of interest in the nascent field of cross-silo federated learning \cite{kairouz2019advances}. 

This idea of automating the data labeling process has been explored in prior works, such as in the dataset construction \cite{collins2008towards} and self-training literatures \cite{Xie_2020_CVPR, scudder1965probability, yarowsky1995unsupervised, riloff2003learning, yalniz2019billion}, but to the best of our knowledge they all employ centralized classifiers, e.g., a single neural network trained on a global dataset, to generate the labels. Our setting is more distributed: we consider the problem of aggregating multiple classifiers trained on distinct subsets of the global dataset. 

\subsection{Ensemble Methods for Multiclass Classification}
The decomposition of multiclass classification problems into binary classification tasks has been studied before, most notably through the ensemble methods of \textit{one-vs.-one} (also known as \textit{all-pairs} or \textit{pairwise coupling}) \cite{Friedman1996AnotherAT, hastie1998classification}, \textit{one-vs.-all} \cite{bishop2006pattern, rifkin2004defense}, and \textit{error-correcting output codes} (ECOCs) \cite{dietterich1994solving}. A more general framework which encapsulates multiclass to binary reductions using margin-based classifiers was proposed in \cite{allwein2000reducing}. One key drawback of one-vs.-all and ECOCs is that they typically require each binary classifier to have access to the entire training set. Importantly, this paper outlines a more general approach to constructing classifiers from not only binary classifiers, but classifiers of arbitrary size which can be trained from only a partial view of the data. Our approach can be viewed as a type of meta or hierarchical ensemble method that combines a diverse set of classifiers which may themselves consist of smaller ensembles. Recall that when $R=2$ (the smaller classifiers are binary), our approach reduces to the one-vs.-one decomposition method.

%% file: appendix.tex
\appendix

\subsection{Proof of Theorem \ref{thm:perfect_UB}}
For the first result, we employ the probabilistic method. Consider a random binary matrix $A \in \{0,1\}^{m \times K}$ with each row chosen independently and uniformly at random from the space of possible weight-$R$ binary vectors of length $K$. We will choose $m$ such that there is a strictly positive probability that $A$ is a fully distinguishing matrix, thereby proving the existence of such a matrix. Let \[P_e = \PP(A \text{ is not fully distinguishing})\] 
and let $E_j$ be the event that two columns of $A$ are different or both equal to $0$ in the $j^\text{th}$ place. 
By the union bound, $P_e$ can be upper bounded by a sum over all ${K \choose 2}$ pairs of columns of $A$ as 
\begin{align*}
P_e &\leq \Sum{i=1}{{K \choose 2}} \PP(\text{the } i^\text{th} \text{ pair of columns of } A \text{ shares no 1 }  \\[-10pt]
	&\hspace{0.5in} \text{ in the same place}) \\
&= {K \choose 2} \Prod{j=1}{m} \PP(E_j) \\
&= {K \choose 2} \PP(E_1)^m \\
&= {K \choose 2} \Big(1-\PP(E_1^c))^m
\end{align*}
where $E_1^c$ is the event that two columns of $A$ are both equal to $1$ in the $1^\text{st}$ place. 
Upon fixing both columns of $A$ to be equal to 1 in the $1^\text{st}$ place, the number of ways to assign the remaining $(R-2)$ ones to the remaining $(K-2)$ spots in the same row is ${K-2 \choose R-2}$. Therefore, 
\begin{align*}
P_e &\leq {K \choose 2} \Bigg(1-\frac{{K-2 \choose R-2}}{{K \choose R}}\Bigg)^m \\
&= \frac{K(K-1)}{2}\cdot \Big(1-\frac{R(R-1)}{K(K-1)}\Big)^m \\
&\leq \frac{K(K-1)}{2} \cdot \exp\Bigg(-m\cdot \frac{R(R-1)}{K(K-1)}\Bigg)
\end{align*}
where the second inequality follows from the fact that $1-x \leq e^{-x}$ for all $x \in \RR$. To ensure that $P_e < 1$, it suffices to choose \[m = \left\lceil\frac{K(K-1)}{R(R-1)}\cdot \log\Bigg(\frac{K(K-1)}{2}\Bigg) + 1\right\rceil.\]

For the second part of the theorem, we simply note that $A$ is fully distinguishing with probability at least $1-\delta$ if $P_e \leq \delta$, from which it follows that 
\[m = \left\lceil \frac{K(K-1)}{R(R-1)} \cdot \log\Bigg(\frac{K(K-1)}{2\delta}\Bigg) \right\rceil\]
is sufficient. 

$\hfill\qed$

\subsection{Proof of Theorem \ref{thm:stat_LB}}
We assume without loss of generality that all classifiers are distinct, as the presence of duplicate classifiers does not reduce the probability of error any more than distinct classifiers do.

\subsubsection{Proof of Lemma \ref{lemma:entropy}}
We expand the conditional entropy of $Y$ given $Z$ as 
\begin{align*}
H(Y\,|\,Z) &= \Sum{k \in [K]}{} \pi(k)H(Y\,|\,Z=k) \\
&= \frac{1}{K}\Sum{k \in [K]}{} H(Y_1,\ldots, Y_m\,|\,Z=k) \\
&= \frac{1}{K}\Sum{k \in [K]}{} \Sum{i \in [m]}{} H(Y_i \,|\, Z=k)
\end{align*}
where the last equality follows from the conditional independence of the classifier predictions given $Z$. Next, note that for each $k \in [K]$ and $i \in [m]$ we have
\[
H(Y_i\,|\, Z=k) =
\begin{cases}
\log R, & \text{ if } k \not\in \calY_i \\
0, & \text{ if } k \in \calY_i 
\end{cases}
\]
as $Y_i$ is deterministic if $k \in \calY_i$, and otherwise equals a class chosen uniformly at random from the $R$ classes in $\calY_i$. Let $N = \Big|\{(k,i) \in [K] \times [m] \, : \, k \not\in \calY_i\}\Big|$, and note that $N = m(K-R)$ since for each $i \in [m]$ there are precisely $K-R$ classes that are not included in $\calY_i$. Continuing from before, we have 
\begin{align*}
H(Y\,|\,Z) &= \frac{1}{K}\Sum{\substack{k\in[K],\, i \in [m] \\ k \not\in \calY_i}}{}\log R \\
&= \frac{1}{K}\cdot N \cdot \log R \\
&= m\cdot \frac{(K-R)}{K}\cdot \log R. 
\end{align*}
$\hfill\qed$

\subsubsection{Main Proof}
If $\calY$ denotes the set of possible values of $Y$, then the entropy of $Y$ can be bounded as \[H(Y) \leq \log |\calY| \leq \log (R^m) = m\log R.\] From Lemma \ref{lemma:entropy}, $H(Y\,|\,Z) = m\cdot \frac{(K-R)}{K}\cdot \log R$. Thus, since the mutual information between $Y$ and $Z$ is given by $I(Y;Z) = H(Y) - H(Y\,|\,Z)$, 
\begin{align}
I(Y;Z) &\leq m \log R - m\cdot \frac{(K-R)}{K}\cdot \log R \nonumber \\
&= m\cdot\frac{R}{K}\log R.  \label{eqn:thm4_1}
\end{align}
Observe that $Z \to Y \to g(Y)$ forms a Markov chain. By the data processing inequality, 
\begin{equation}\label{eqn:thm4_2} 
I(Y;Z) \geq I(Z;g(Y)) = H(Z) - H(Z\,|\,g(Y)).
\end{equation}
Since $Z$ is uniformly distributed over the $K$ classes by assumption, we have $H(Z) = \log K$. If $H(P_e)$ denotes the binary entropy corresponding to $P_e = \PP(g(Y) \neq Z)$, then $H(P_e) \leq \log(2)$. By Fano's inequality, combined with the assumption that $P_e \leq \epsilon$, we therefore have
\begin{align}
H(Z\,|\,g(Y)) &\leq H(P_e) + P_e\log(K-1) \nonumber \\
&\leq \log(2) + \epsilon \log K. \label{eqn:thm4_3}
\end{align}
Combining (\ref{eqn:thm4_1}), (\ref{eqn:thm4_2}), and (\ref{eqn:thm4_3}) yields 
\begin{align*}
m\cdot\frac{R}{K}\log R &\geq H(Z)-H(Z\,|\,g(Y)) \\
&\geq \log K - (\log(2)+\epsilon \log K) \\
&= (1-\epsilon)\log K - \log(2)
\end{align*}
and dividing both sides of the inequality by $\frac{R}{K}\log R$ and taking the ceiling gives the final result. 
$\hfill\qed$

\subsection{Proof of Theorem \ref{thm:stat_UB}}
The probability of error can first be bounded as 
\begin{align}
P_e &= \frac{1}{K}\Sum{k=1}{K}\PP(g(Y) \neq k \, | \, Z=k) \nonumber \\
&= \PP(g(Y) \neq 1 \, | \, Z = 1) \label{eqn:sym} 
 \\ 
&\leq \PP\Big(\exists j \in \{2,\ldots,K\} \, : \, \calL(Y;j) \geq \calL(Y;1) \, | \, Z=1\Big) \nonumber \\
&\leq \Sum{j=2}{K} \PP\Big(\calL(Y;j) \geq \calL(Y;1)\, | \, Z=1\Big) \label{eqn:union} 
\\
&\leq K\cdot \PP\Big(\calL(Y;2) \geq \calL(Y;1)\, | \, Z=1\Big).\label{eqn:ml}
\end{align}
where (\ref{eqn:sym}) follows from the symmetry in the random classifier construction, and (\ref{eqn:union}) uses the union bound. In (\ref{eqn:ml}), we again use the symmetry in the construction, and $\PP$ represents the randomness in both the construction and in $Y$. 

Let $Y'$ be a random output drawn conditionally on $Z=1$. Consider the event $E: \calL(Y';2) \geq \calL(Y';1)$, and further define the events $B_i: 1\not\in \calY_i, 2 \in \calY_i$ and $C_i: Y'_i = 2$ for each $i \in [m]$. We claim that for any $i \in [m]$, \[B_i \cap C_i^c \Rightarrow E^c,\] which can be verified through the following argument. Fix $i \in [m]$, and note that the event $B_i \cap C_i^c$ means that both $1\not\in \calY_i, 2 \in \calY_i$ and $Y_i' \neq 2$. Under our classifier model, we have $\PP(Y_i = 2\,|\, Z=2) = 1$, or equivalently $\PP(Y_i \neq 2 \,|\, Z=2) = 0$. It follows that $\calL(Y';2) = 0$, whereas $\calL(Y';1) > 0$ since $Y'$ was drawn from $\calS_1$ (the output set of class 1), thus proving the claim.

As a consequence of the above claim, we have that \[\bigcup_{i=1}^m (B_i \cap C_i^c) \subseteq E^c\] and hence, by standard set-theoretic arguments, 
\begin{align*}
E &\subseteq \Bigg(\bigcup_{i=1}^m (B_i \cap C_i^c)\Bigg)^c 
= \bigcap_{i=1}^m (B_i^c \cup C_i) 
= \bigcap_{i=1}^m \Big(B_i^c \cup (C_i \cap B_i)\Big).
\end{align*}
It follows that 
\begin{align*}
P_e &\leq K\cdot \PP(E) \\
&\leq K\cdot \PP\Big(\bigcap_{i=1}^m \Big(B_i^c \cup (C_i \cap B_i)\Big)\Big) \\
&= K\cdot \prod_{i=1}^m\PP(B_i^c \cup (C_i \cap B_i)) \\
&= K\cdot \PP(B_1^c \cup (C_1 \cap B_1))^m  \\
&= K \cdot \Big(\PP(B_1^c) + \PP(C_1\,|\,B_1)\PP(B_1)\Big)^m \\
&= K\cdot \Big(1-\PP(B_1) + \PP(C_1\,|\,B_1)\PP(B_1)\Big)^m.
\end{align*}

It now remains to compute $\PP(B_1)$ and $\PP(C_1\,|\,B_1)$. First, note that when conditioned on $B_1 : 1\not\in \calY_1, 2 \in \calY_1$, the probability of $C_1: Y_1' = 2$ is exactly $1/R$ (recall that $Y'$ is a random output conditioned on $Z=1$). Hence, 
\[\PP(C_1\,|\,B_1) = \frac{1}{R}.\]
Second, under the random construction in which classifiers are selected independently and uniformly at random from the set of all size-$R$ classifiers, the probability of $B_1$ is proportional to the number of ways to choose the remaining $(R-1)$ classes from the remaining $(K-2)$ total classes (since we are constraining the classifier to contain class 2 but not class 1). That is, 
\[\PP(B_1) = \frac{{K-2 \choose R-1}}{{K \choose R}} = \frac{R(K-R)}{K(K-1)}.\]
Continuing the previous bound, we now have 
\begin{align*}
P_e &\leq K\cdot \Big(1-\frac{R(K-R)}{K(K-1)} + \frac{K-R}{K(K-1)}\Big)^m \\
&= K\cdot \Big(1-\frac{(K-R)(R-1)}{K(K-1)}\Big)^m \\
&\leq K\cdot \exp\Big(-m\cdot\frac{(K-R)(R-1)}{K(K-1)}\Big)
\end{align*}
where the final inequality uses the fact that $1-x \leq e^{-x}$. 
To ensure that $P_e \leq \epsilon$, we need
\[m \geq \frac{K(K-1)}{(K-R)(R-1)}\cdot \log\Big(\frac{K}{\epsilon}\Big)\]
and thus it suffices to have 
\[m = \left\lceil \frac{K(K-1)}{(K-R)(R-1)}\cdot \log\Big(\frac{K}{\epsilon}\Big) \right\rceil. \]
$\hfill\qed$